\documentclass{article}

\usepackage[utf8]{inputenc} 
\usepackage[T1]{fontenc}    
\usepackage{amsfonts}       
\usepackage{nicefrac}       
\usepackage{microtype}      
\usepackage{amssymb}
\usepackage{algorithm}
\usepackage{caption}
\usepackage{amsmath}
\usepackage{amsthm}
\usepackage{graphicx}
\usepackage{subfigure}
\usepackage{tabularx}
\usepackage{pifont}
\usepackage{algpseudocode}
\usepackage{bm}
\usepackage{array}
\usepackage{balance}
\usepackage{amsthm}
\usepackage{amsmath}
\usepackage{amssymb}
\usepackage{multirow}
\usepackage{fullpage}
\usepackage{color}
\usepackage{url} 
\usepackage{enumitem}

\usepackage{natbib}
\allowdisplaybreaks

\newcommand{\Tr}{^{\rm \top}}

\newcommand{\tr}{{\rm tr}}

\newcommand{\Rcal}{{\mathcal{R}}}
\newcommand{\Acal}{{\mathcal{A}}}

\newcommand{\Lcal}{{\mathcal{L}}}

\newcommand{\e}{{\bf e}}

\renewcommand{\u}{{\bf u}}
\renewcommand{\v}{{\bf v}}

\newcommand{\x}{{\bf x}}
\newcommand{\y}{{\bf y}}
\newcommand{\z}{{\bf z}}

\newcommand{\A}{{\bf A}}

\newcommand{\Dcal}{\mathcal{D}}

\newcommand{\G}{{\bf G}}
\newcommand{\Gcal}{{\mathcal{G}}}

\newcommand{\I}{{\bf I}}

\newcommand{\M}{{\bf M}}

\renewcommand{\P}{{\bf P}}

\newcommand{\V}{{\bf V}}

\newcommand{\W}{{\bf W}}

\newcommand{\Ocal}[1]{{\mathcal{O}\left( #1  \right)}}

\newcommand{\X}{{\bf X}}
\newcommand{\Xcal}{{\mathcal{X}}}

\newcommand{\bLambda}{\mathbf{\Lambda}}

\newcommand{\1}{{\bf 1}}
\newcommand{\0}{{\bf 0}}

\newcommand{\argmin}{\operatornamewithlimits{argmin}}

\newcommand{\lrincir}[1]{\left( #1 \right)}

\newcommand{\norm}[1]{\lVert#1\rVert}
\newcommand{\lrnorm}[1]{\left\lVert#1\right\rVert}
\newcommand{\lrangle}[1]{\left\langle#1 \right\rangle}

\newcommand{\EE}{\mathop{\mathbb{E}}}
\newcommand{\RR}{\mathbb{R}}

\newcommand{\refabovecir}[2]{\displaystyle_{#1}^{#2}}

 \newtheorem{Theorem}{\bf{Theorem}}
 
 \newtheorem{Lemma}{\bf{Lemma}}
 
 \newtheorem{Corollary}{\bf{Corollary}}
 
 \newtheorem{Assumption}{\bf{Assumption}}

\title{Decentralized Online Learning: Take Benefits from Others' Data without Sharing Your Own to Track Global Trend}

\author{
  $^{1,2}$Yawei Zhao\footnote{Equal contribution.}, $^2$Chen Yu$^\ast$, $^{3}$Peilin Zhao, $^{2}$Hanlin Tang, $^{4}$Shuang Qiu, and $^{2}$Ji Liu \\
  $^1$National University of Defense Technology, Changsha, China\\
  $^2$University of Rochester, Rochester, USA \\
  $^3$Tencent AI Lab, Shenzhen, China\\
  \texttt{zhaoyawei@nudt.edu.cn}, 
  \texttt{yuchen92328@gmail.com}, \\\texttt{peilinzhao@hotmail.com}, \texttt{htang14@ur.rochester.edu}, \texttt{ji.liu.uwisc@gmail.com} \\
}

\begin{document}

\maketitle

\begin{abstract}
Decentralized Online Learning (online learning in decentralized networks) attracts more and more attention, since it is believed that Decentralized Online Learning can help the data providers cooperatively better solve their online problems without sharing their private data to a third party or other providers.
Typically, the cooperation is achieved by letting the data providers exchange their models between neighbors, e.g., recommendation model. However, the best regret bound for a decentralized online learning algorithm is $\Ocal{n\sqrt{T}}$, where $n$ is the number of nodes (or users) and $T$ is the number of iterations. This is clearly insignificant since this bound can be achieved \emph{without} any communication in the networks. This reminds us to ask a fundamental question: \emph{Can people really get benefit from the decentralized online learning by exchanging information?}
In this paper, we studied when and why the communication can help the decentralized online learning to reduce the regret.
Specifically, each loss function is characterized by two components: the adversarial component and the stochastic component.
Under this characterization, we show that decentralized online gradient (DOG) enjoys a regret bound  $\Ocal{\sqrt{n^2TG^2 + n T \sigma^2}}$, where $G$ measures the magnitude of the adversarial component in the private data (or equivalently the local loss function) and $\sigma$ measures the randomness within the private data. This regret suggests that people can get benefits from the randomness in the private data by exchanging private information. Another important contribution of this paper is to consider the dynamic regret -- a more practical regret to track users' interest dynamics. Empirical studies are also conducted to validate our analysis.
\end{abstract}

\section{Introduction}
\label{sect_introduction}
Decentralized online learning receives extensive attentions in recent years~\citep{8015179Shahram,Kamp:2014:CDO,Koppel-8352032,Zhang2018,pmlr-v70-zhang17g,Xu2015,tcns-7353155,cdc-7798923,acc-7172037,tcns-7479495,Benczur:2018ww,tkde-6311406}. 
It assumes that computational nodes in a network can communicate between neighbors to minimize an overall cumulative regret.
Each computational node, which could be a user in practice, will receive a stream of online losses that are usually determined by a sequence of examples that arrive sequentially. 
Formally, we can denote $f_{i,t}$  as the loss received by the $i$-th computational node among the networks at the $t$-th iteration. 
The goal of decentralized online learning usually is to minimize its static regret, which is defined as the difference between the cumulative loss (the sum of all the online loss over all the nodes and steps ) suffered by the learning algorithm and that of the best model which can observe all the loss functions beforehand. 

Decentralized online learning attracts more and more attentions recently, mainly because it is believed by the community that it enjoys the following advantages for real-world large-scale applications:
\begin{itemize}[leftmargin=*]
\item ({\bf Utilize all computational resource}) It can utilize the computational resource (of edging devices) by avoiding collecting all the loss functions (or equivalently data) to one central node and putting all computational burden on a single node. 
\item ({\bf Protect data privacy}) It can help many data providers collaborate to better minimize their cumulative loss, while at the same time protecting the data privacy as much as possible. 
\end{itemize}
However, the current theoretical study does not explain why people need to use decentralized online learning, since the currently best regret result for decentralized online learning $\Ocal{n\sqrt{T}}$ for convex loss functions~\citep{6760092,tkde-6311406}) is equal to the overall regret if each node (user) only runs local online gradient without any communication with others\footnote{$n$ is the number of nodes or users and $T$ is the total number of iterations. The regret of an online algorithm is $\Ocal{\sqrt{T}}$ for convex loss functions~\citep{Hazan2016Introduction,ShalevShwartz:2012dz}. Therefore, the overall regret is $n\sqrt{T}$ if all users do not communicate.}. Therefore, this reminds us to ask a fundamental question: \emph{Can people really get benefit with respect to the regret from the decentralized online learning by exchanging information?}

In this paper, we mainly study when can the communication really help decentralized online learning to minimize its regret. Specifically, we distinguish two components in the loss function $f_{i,t}$: the adversary component and the stochastic compoent. Then we prove that decentralized online gradient can achieve a static regret bound of $\Ocal{\sqrt{n^2TG^2+ nT\sigma^2}}$ ($G$ represents the bound of gradient. $\sigma$ measures the randomness of the private data), where the first component of the bound is due to the adversary loss while the second component is due to the stochastic loss.
Moreover, if a dynamic sequence of models with a budget $M$ is used as the reference points, the dynamic regret of the decentralized online gradient is $\Ocal{\sqrt{\lrincir{n^2TG^2 + nT\sigma}(M+1)}}$. The dynamic regret is a more suitable performance metric for real-world applications where the optimal model changes over time, such as people's favorite style of pop musics usually change along with time as the global environment. This shows the communication can help to minimize the stochastic losses, rather than the adversary losses. This result is further verified empirically by extensive experiments.

\textbf{Notations.} In the paper, we make the following notations.
\begin{itemize}[leftmargin=*]
\item For any $i\in[n]$ and $t\in[T]$, the random variable $\xi_{i,t}$ is subject to a distribution $D_{i,t}$, that is, $\xi_{i,t} \sim D_{i,t}$. A set of random variables $\Xi_{n,T}$ and their corresponding distributions are defined by
\begin{align}
\nonumber
\Xi_{n,T} =  \{ \xi_{i,t} \}_{1\le i \le n, 1 \le t \le T}, {~~~~}\Dcal_{n,T} =  \{ D_{i,t} \}_{1\le i \le n, 1 \le t \le T},
\end{align} respectively. For math brevity, we use the notation $\Xi_{n,T} \sim \Dcal_{n,T}$ to represent that $\xi_{i,t} \sim D_{i,t}$ holds for any $i\in[n]$ and $t\in[T]$. $\EE$ represents mathematical expectation.
\item `$\nabla$' represents gradient operator. `$\lrnorm{\cdot}$' represents the $\ell_2$ norm in default.  `$\lesssim$' represents ``less than equal up to a constant factor". `$\Acal$' represents the set of all online algorithms. `$\1$' and `$\0$' represent all the elements of a vector is $1$ and $0$, respectively.
\end{itemize}

\section{Related work}
\label{sect_related_work}
Online learning has been studied for decades of years. An online convex optimization method can achieve a static regret bound of order $\Ocal{\sqrt{T}}$ and $\Ocal{\log T}$ for convex and strongly convex loss functions, respectively \citep{Hazan2016Introduction,ShalevShwartz:2012dz,introduction-online-optimization}. 

\textbf{Decentralized online learning.} Online learning in a decentralized network has been studied in \citep{8015179Shahram,Kamp:2014:CDO,Koppel-8352032,Zhang2018,pmlr-v70-zhang17g,Xu2015,tcns-7353155,cdc-7798923,acc-7172037,tcns-7479495,Benczur:2018ww,tkde-6311406}.  \citet{8015179Shahram} provides a dynamic regret (defined in Eq. \eqref{equa_definition_our_regret}) bound of $\Ocal{n\sqrt{nTM}}$ for decentralized online mirror descent, where $n$, $T$, and $M$ represent the number of nodes in the newtork, the number of iterations, and the budget of dynamics, respectively.  When the Bregman divergence in the decentralized online mirror descent is chosen appropriately, the decentralized online mirror descent becomes identical to the decentralized online gradient descent. 
In this paper, we achieve a better dynamic regret bound of  $\Ocal{n\sqrt{TM}}$ for a decentralized online gradient descent method, which mainly benefits from a better bound of network error (see Lemma \ref{Lemma_x_variance_norm_square}). Moreover, \citet{Kamp:2014:CDO} presents a static regret of  $\Ocal{\sqrt{nT}}$ for decentralized online prediction. However,  it assumes that all the loss functions are generated from an unknown identical distribution,
 this assumption is too strong to be practical in the dynamic environment and be applied for a general online learning task. Additionally, many decentralized online optimization methods are proposed, for example, decentralized online multi-task learning \citep{Zhang2018}, decentralized online ADMM \citep{Xu2015}, decentralized online gradient descent \citep{tcns-7353155}, decentralized continuous-time online saddle-point method \citep{cdc-7798923}, decentralized online  Nesterov's primal-dual method \citep{acc-7172037,tcns-7479495}, and online distributed dual averaging \citep{6760092}.
However, these previous work only studied  the static regret bounds $\Ocal{\sqrt{T}}$ of the decentralized online learning algorithms, while they did not provide any theoretical analysis for dynamic environments. 
Besides,  \citet{tkde-6311406} provides necessary and sufficient conditions to preserve privacy for decentralized online learning methods, which be studied to extend our method in our future work.

\textbf{Dynamic regret.} The dynamic regret of online learning algorithms  has been widely studied for decades~\citep{Zinkevich:2003,Hall:2015ct,Hall:2013vr,Jadbabaie:2015wg,Yang:2016ud,Bedi:2018te,Zhang:2016wl,Mokhtari:2016jz,Zhang:2018tu,Gyorgy:2016,NIPS2016_6536,Zhao:2018wx}.  
The first dynamic regret is defined as $\sum_{t=1}^T \lrincir{ f_{t}(\x_{t}) - f_{t}(\x_t^\ast) }$ subject to $\sum_{t=1}^{T-1} \lrnorm{\x_{t+1}^\ast - \x_t^\ast} \le M$ where $M$ is a budget for the change over the reference points~\citep{Zinkevich:2003}.
For this definition, the online gradient descent can achieve a dynamic regret of $\Ocal{\sqrt{TM}+\sqrt{T}}$, by selecting an appropriate learning rate. 
Later, other types of dynamic regrets are also introduced, by using different types of reference points.
For example, \citet{Hall:2015ct,Hall:2013vr} choose the reference points $\{\x_t^{\ast}\}_{t=1}^T$ satisfying $\sum_{t=1}^{T-1} \lrnorm{\x_{t+1}^\ast - \Phi(\x_t^\ast)} \le M$, where $\Phi(\x_t^\ast)$ is the predictive optimal model. When the function $\Phi$ predicts accurately, the budget $M$ can decrease significantly so that the dynamic regret effectively decreases. \citet{Jadbabaie:2015wg,Yang:2016ud,Bedi:2018te,Zhang:2016wl,Mokhtari:2016jz,Zhang:2018tu} chooses the reference points $\{\y_t^{\ast}\}_{t=1}^T$ with $\y_t^{\ast} = \argmin_{\z\in\Xcal} f_t(\z)$, where $f_t$ is the loss function at the $t$-th iteration. \citet{Gyorgy:2016} provides a new analysis framework, which achieves $\Ocal{\sqrt{TM}}$ dynamic regret\footnote{\citet{Gyorgy:2016} uses the notation of ``shifting regret" instead of ``dynamic regret". In the paper, we keep using ``dynamic regret" as used in most previous literatures. } for all the above reference points. Recently, the lower bound of the dynamic regret was shown to be $\Omega\lrincir{\sqrt{TM}}$ \citep{NIPS2018_7407, Zhao:2018wx}, which indicates that the above algorithms are optimal in terms of dynamic regret.
In this paper, we propose a new definition of dynamic regret, which covers all the previous ones as special cases. 

In some literatures, the regret in a dynamic environment is measured by the number of changes of the reference points over time, which is usually termed as shifting regret or tracking regret \citep{Herbster1998,Gyorgy:2005wo,Gyorgy:2012wa,Gyorgy:2016,Mourtada:2017vn,JMLR:v17:13-533,NIPS2016_6536,cesabianchi:hal,pmlr-v84-mohri18a,pmlr-v54-jun17a}. Both the shifting regret and the tracking regret are usually studied in the setting of ``learning with expert advice" while the dynamic regret is more often studied in the general setting of online learning.

\section{Problem formulation}
In decentralized online learning, the topological structure of the network can be represented by an undirected graph $\Gcal=(\text{nodes:} [n], \text{edges:} E)$ with vertex set $[n]=\{1,\ldots,n\}$ and edge set $E\subset [n]\times [n]$. 
In real applications, each node $i\in [n]$ is associated with a separate learner, for example an mobile device of one user, which maintains a local predictive model.
Users would like to cooperatively better minimize their regret without sharing their private data. They typically share private models to their neighbors (or friends), which are directly adjacent nodes in $\Gcal$ for each node. 

Let $\x_{i, t}$ denote the local model for user $i$ at iteration $t$. 
In iteration $t$ user $i$ predicts the local model $\x_{i,t}$ for an unknown loss, and then receives the loss $f_{i,t}(\cdot; \xi_{i,t})$. As a result, the decentralized online learning algorithm suffered a instantaneous loss $f_{i,t}(\x_{i,t}; \xi_{i,t})$.
$\xi_{i,t}$'s are independent to each other in terms of $i$ and $t$, charactering the \emph{stochastic} component in the function $f_{i,t}(\cdot; \xi_{i,t})$, while the subscripts $i$ and $t$ of $f$ indicate the \emph{adversarial} component, for example, the user's profile, location, local time, and etc. The stochastic component in the function is usually caused by the potential relation among local models. For example, users' perference to music may be impacted by a popular trend in the Internet at the same time.

To measure the efficacy of a decentralized online learning algorithm $A\in\Acal$, a commonly used performance measure is the  \emph{static} regret which is defined as 
\begin{align}
\nonumber
\widetilde{\Rcal}_T^{A} :=  \EE_{ \Xi_{n,T} \sim \Dcal_{n,T} }  \left[\sum_{i=1}^n\sum_{t=1}^T \lrincir{f_{i,t}(\x_{i,t};\xi_{i,t}) - f_{i,t}(\x^\ast;\xi_{i,t}} \right],
\end{align} where $\x^\ast=\arg \min_\x \EE_{ \Xi_{n,T} \sim \Dcal_{n,T} }  \sum_{i=1}^n\sum_{t=1}^T f_{i,t}(\x;\xi_{i,t}) $. The static regret essentially assumes that the optimal model would not change over time. However, in many practical online learning application scenarios, the optimal model may evolve over time. For example, when we want to conduct music recommendation to a user, user's preference to music may change over time as his/her situation.  Thus, the optimal model $\x^\ast$ should change over time. It leads to the dynamics of the optimal recommendation model. Therefore, for any online algorithm $A\in\Acal$, we choose to use the \emph{dynamic} regret as the metric:
\begin{align}
\label{equa_definition_our_regret}
\Rcal_T^{A} :=  \EE_{ \Xi_{n,T} \sim \Dcal_{n,T} }  \left[\sum_{i=1}^n\sum_{t=1}^T f_{i,t}(\x_{i,t};\xi_{i,t}) \right] - \min_{\{\x_t^\star\}_{t=1}^T \in \mathcal{L}_{M}^T}  \EE_{ \Xi_{n,T} \sim \Dcal_{n,T} }\left[\sum_{i=1}^n\sum_{t=1}^T f_{i,t}(\x_t^\ast;\xi_{i,t}) \right],
\end{align}
where $\mathcal{L}_M^T$ is defined by $\Lcal_{M}^T = \left\{\{\z_t\}_{t=1}^T : \sum\limits_{t=1}^{T-1}\|\z_{t+1}-\z_t\|\le M \right\}$. $\mathcal{L}_M^T$ restricts how much the optimal model may change over time. Obviously, $\Rcal_T^{A}$ degenerates to $\widetilde{\Rcal}_T^{A}$ when $M=0$.

\section{Decentralized Online Gradient (DOG) algorithm} \label{sec:algorithm}
In the section, we introduce the DOG algorithm, followed by the analysis for the dynamic regret.

\begin{algorithm}[!]
   \caption{\textsc{DOG}: Decentralized Online Gradient method.}
   \label{algo_DOG}
   \begin{algorithmic}[1]
   \Require Learning rate $\eta$, number of iterations $T$, and the confusion matrix $\W$.    
  \State Initialize $\x_{i,1} = \0$ for all $i\in [n]$;    
   \For {$t=1,2, ..., T$}
            \State $\slash\slash$ For all users (say the $i$-th node $i\in[n]$)
                        \State {\bf Query} the neighbors' local models $\{\x_{j,t}\}_{j\in \text{user $i$'s neighbor set}}$;
            \State Apply the local model and suffer loss $f_{i,t}(\x_{i,t};\xi_{i,t})$;
            \State {\bf Compute} the gradient $\nabla f_{i,t}(\x_{i,t};\xi_{i,t})$;
            \State {\bf Update} the local model by $ \x_{i,t+1} = \sum_{j\in \text{user $i$'s neightbours}} \W_{i,j}\x_{j,t} - \eta \nabla f_{i,t}(\x_{i,t};\xi_{i,t})$;
       \EndFor
   \end{algorithmic}
\end{algorithm}

\subsection{Algorithm description}
In the DOG algorithm, users exchange their local models periodically. In each iteration, each user runs the following steps:  1) Query the local models from his/her all neighbors; 2) Apply the local model to $f_{i,t}(\cdot; \xi_{i,t})$ and compute the gradient; 3) Update the local model by taking average with neighbors' models followed by a gradient step.
The detailed description of the DOG algorithm can be found in Algorithm~\ref{algo_DOG}. $\W\in\RR^{n\times n}$ is the confusion matrix defined on an undirected graph $\Gcal=(\text{nodes:} [n], \text{edges:} E)$. It is assumed to be a \textit{doubly stochastic} matrix \citep{7903733,8320863,Yuan:2016ur}, but not necessarily symmetric. Given a decentralized network $\Gcal$, there are many approaches to obtain a doubly stochastic $\W$, for example, Sinkhorn matrix scaling \citep{knigh_sinkhorm_scale}. The following are two naive ways to construct such a doubly stochastic matrix $\W$. Let $N_i$ be the number of user $i$'s neighbors (exclusive itself), and $N_{\max}:=\max_i N_i$. We can obtain a doubly stochastic matrix by:
1) $\W_{i,j}=0$, if $i$ and $j$ are not connected ($i\neq j$) in $E$;
2) $\W_{i,j}={1\over n}$, if $j$ are $i$ are connected ($i\neq j$) in $E$; 
3) $\W_{i,i}=1-{N_i \over n}$. When $N_{\max}:=\max_{i\in [n]} N_i$ is known, an alternative method is: 1) $\W_{i,j}=0$, if $i$ and $j$ are not connected ($i\neq j$) in $E$; 
2) $\W_{i,j}={1\over N_{\max}+1}$, if $j$ are $i$ are connected ($i\neq j$) in $E$;
3) $\W_{i,i}=1-{N_i \over N_{\max}+1}$.
To take a closer look at the algorithm’s updating rule, we define $\bar{\x}_t = \frac{1}{n}\sum_{i=1}^n \x_{i,t} $. It is not hard to verify that $\bar{\x}_{t+1} =  \bar{\x}_t -  \frac{\eta}{n}\sum_{i=1}^n \nabla f_{i,t}(\x_{i,t};\xi_{i,t})$ The detailed proofs are provided in Lemma \ref{Lemma_average_update_rule} (See Supplemental Materials.)



\subsection{Dynamic regret Analysis}  
Next we show the dynamic regret of DOG in the following. Before that, we first make some common assumptions used in our analysis. 

\begin{Assumption}
\label{assumption_bounded_gradient_domain}
We make following assumptions throughout this paper:
\begin{itemize}[leftmargin=*]
\item For any $i\in[n]$, $t\in[T]$, and $\x$, there exist constants $G$ and $\sigma$ such that $\lrnorm{\EE_{ \xi_{i,t} \sim D_{i,t} } \nabla f_{i,t}(\x;\xi_{i,t})} \le  G$, and 
$\EE_{ \xi_{i,t} \sim D_{i,t} } \lrnorm{\nabla f_{i,t}(\x; \xi_{i,t}) - \EE_{ \xi_{i,t} \sim D_{i,t} } \nabla f_{i,t}(\x)}^2 \le \sigma^2$.
\item For given vectors $\x$ and $\y$, we assume $\lrnorm{\x-\y}^2 \le R$. Besides, for any $i\in[n]$ and $t\in[T]$, we assume $f_{i,t}$ is convex, and has $L$-Lipschitz gradient. 
\item Let $\W$ be doubly stochastic and $\rho := \lrnorm{\W - \frac{\1\1\Tr}{n}}$. Assume $\rho <1$.
\end{itemize}
\end{Assumption}  
$G$ essentially gives the upper bound for the adversarial component in $f_{i,t}(\cdot; \xi_{i,t})$. The stochastic component is bounded by $\sigma^2$. Note that if there is no stochastic component, $G$ is nothing but the upper bound of the gradient like the setting in many online learning literature. It is important for our analysis to split these two components, which will be clear very soon. 

The last assumption about $\W$ is an essential assumption for the decentralized setting. The largest eigenvalue for a doubly stochastic matrix is $1$. $1-\rho$ measures how fast the information can propagate within the network (the larger the faster). Now we are ready to present the dynamic regret for DOG.
\begin{Theorem}
\label{theorem_regret_upper_bound}
Denote constants $C_0$, $C_1$, and $C_2$ by $C_0 := \frac{2L (G^2 + \sigma^2)}{(1-\rho)^2}$, $C_1 := \frac{4L^2 (G^2 + \sigma^2)}{(1-\rho)^2}$, $C_2 := 2+\frac{1}{1-\rho}$. Choosing $\eta>0$ in Algorithm \ref{algo_DOG}, under Assumption \ref{assumption_bounded_gradient_domain} we have
\begin{align}
\nonumber
\Rcal_T^{\textsc{DOG}} \le \eta T \sigma^2 +  C_0 n T \eta^2 + C_1 n T \eta^3 + \frac{n}{2\eta}\lrincir{ 4\sqrt{R}M + R  } + C_2 n\eta T G^2.
\end{align}
\end{Theorem}

By choosing an appropriate learning rate $\eta$, we obtain sublinear regret as follows.
\begin{Corollary}
\label{corollary_regret_upper_bound}
Choosing $\eta = \sqrt{\frac{(1-\rho) \lrincir{nM\sqrt{R} + nR}}{ nTG^2 + T\sigma^2 }}$ in Algorithm \ref{algo_DOG}, under Assumption \ref{assumption_bounded_gradient_domain} we have 
\[
\Rcal_T^{\textsc{DOG}} \lesssim  \sqrt{\frac{T\lrincir{M+\sqrt{R}}(n^2G^2 + n\sigma^2)}{1-\rho}} + \frac{n^2\lrincir{M+\sqrt{R}}}{1-\rho} + \frac{n^{\frac{5}{2}} \lrincir{M+\sqrt{R}}}{\sqrt{(1-\rho)T}}.
\]
\end{Corollary}       
For simpler discussion, let us treat $R$, $G$, and $1-\rho$ as constants. The dynamic regret can be simplified into $\Ocal{\sqrt{\lrincir{n^2TG^2 + nT\sigma^2}(M+1)}}$. If $M=0$, the dynamic regret degenerates the static regret $\Ocal{\sqrt{n^2TG^2 + n T \sigma^2}}$.
More discussion is conducted in the following aspects.
\begin{itemize}[leftmargin=*]
\item ({\bf Tightness.}) To see the tightness, we consider a few special cases:
\begin{itemize}[leftmargin=*]
\item ($\sigma= 0$ and $n=1$.) It degenerates to the vanilla online learning setting but with dynamic regret. The implied static regret $\Ocal{\sqrt{TM}}$ is consistent with the dynamic regret result in \citet{Zhao:2018wx}, which is proven to be optimum. 
\item ($M=0$.) It degenerates to the static regret $\Ocal{\sqrt{n^2TG^2 + nT\sigma^2}}$. When $G<\sigma/\sqrt{n}$, that is, the stochastic component dominates the adversarial component, the static regret $\Ocal{\sqrt{nT}\sigma}$ implies the average regret ${1\over n}\Rcal_T^{\textsc{DOG}}$ to be $\sigma \sqrt{T} / \sqrt{n}$, and the convergence rate (from the stochastic optimization perspective) to be $\sigma / \sqrt{nT}$. It is consistent with the ergodic convergence rate in \citet{Tang:2018un}, and improves the non-ergodic result $\sigma / \sqrt{T}$ in \citet{Nedic:2009ba}\footnote{See the non-ergodic convergent rate in Proposition $3(b)$ in \citet{Nedic:2009ba}. }.
\end{itemize}
\item ({\bf Insight.}) Setting $M=0$, we obtain the static regret $\Ocal{\sqrt{n^2TG^2 + n T \sigma^2}}$. Consider the baseline that all users do not communicate but only run local online gradient. It is not hard to verify that the static regret for this baseline approach is $\Ocal{\sqrt{n^2TG^2 + n^2 T \sigma^2}}$.  Comparing with the baseline, the improvement of our new bound is only on the stochastic component. Denote that $G$ measures the magnitude of the adversarial component and $\sigma$ measures the stochastic component. This result reveals an important observation that \emph{the communication does not really help improve the adversarial component, only the stochastic component can benefit from the communication.} This observation makes quite sense, since if the users' private data are totally arbitrary, there is no reason they can benefit to each other by exchanging anything.
\item ({\bf Improve existing dynamic regret in decentralized setting.}) If we also treat $G$ and $\sigma$ to be constants, our analysis leads to a dynamic regret $\Ocal{n\sqrt{MT}}$. \citet{8015179Shahram} only considers the adversary loss, and provides $\Ocal{n^{\frac{3}{2}}\sqrt{{MT}}}$ regret for DOG. Compared with their result, our regret enjoys the state-of-the-art dependence on $T$ and $M$, and meanwhile improves the dependence on $n$.
\end{itemize} 

Next we discuss how close all local models $\x_{i,t}$'s close to their average at each time. The following result suggests that $\x_{i,t}$'s are getting closer and closer over iterations.
\begin{Theorem}
\label{theorem_local_models_closer}
Denote $\bar{\x}_t = \frac{1}{n}\sum_{i=1}^n \x_{i,t}$.
Choosing $\eta = \sqrt{\frac{(1-\rho) \lrincir{nM\sqrt{R} + nR}}{ nTG^2 + T\sigma^2 }} $ in Algorithm \ref{algo_DOG}, under Assumption \ref{assumption_bounded_gradient_domain} we have  
\[
{1\over nT}\left[\EE_{ \Xi_{n,T} \sim \Dcal_{n,T} } \sum_{i=1}^n\sum_{t=1}^T \lrnorm{\x_{i,t} - \bar{\x}_t}^2\right] \lesssim \frac{n(M + \sqrt{R})}{(1-\rho)T}.
\]
\end{Theorem}
The result suggests that $\x_{i,t}$ approaches to $\bar{\x}_t$ roughly in the rate $\Ocal{1/{T}}$ (treat $M$ and $\rho$ as constants.), which is faster than the convergence of the averaged regret $\Ocal{1/\sqrt{T}}$ from Corollary~\ref{corollary_regret_upper_bound}.

\section{Empirical studies}

We consider online logistic regression with squared $\ell_2$ norm regularization. In the task, $f_{i,t}(\x;\xi_{i,t}) = \log\lrincir{1+\exp(-\y_{i,t}\A_{i,t}\Tr \x)} + \frac{\gamma}{2}\lrnorm{\x}^2$, where $\gamma = 10^{-3}$ is a given hyper-parameter. $\xi_{i,t}$ is the randomness of the function $f_{i,t}$, which is caused by the randomness of data in the experiment. Under this setting, we compare the performance of the proposed Decentralized Online Gradient method (DOG) and that of the Centralized Online Gradient method (COG). 

The dynamic budget $M$ is fixed as $10$ to determine the space of reference points. The learning rate $\eta$ is tuned to be optimal for each dataset separately. We evaluate the learning performance by measuring the \textit{average loss}: $\frac{1}{nT}\sum_{i=1}^n\sum_{t=1}^T f_{i,t}(\x_{i,t};\xi_{i,t})$, instead of using the dynamic regret $\EE_{\Xi_{n,T}\sim \Dcal_{n,T}}\sum_{i=1}^n\sum_{t=1}^T \lrincir{ f_{i,t}(\x_{i,t};\xi_{i,t}) - f_{i,t}(\x_t^{\ast}) }$ directly, since the optimal reference point $\{\x_t^\ast\}^T_{t=1}$ is the same for both DOG and COG.   We perform emperical evaluation on a toy dataset and several real-world datasets, whose details are presented as follows.

\textbf{Synthetic Data.} 
For the $i$-th node, a data matrix  $\A_i\in R^{10\times T}$ is generated, s.t. $\A_i=\beta\tilde{\A}_i+(1-\beta)\hat{\A}_i$, where $\tilde{\A}_i$ represents the adversary part of data, and $\hat{\A}_i$ represents the stochastic part of data. $\beta$ with $0< \beta <1$ is used to make balance between the adversary and stochastic components. A large $\beta$ represents the adversary component is significant, and the stochastic component becomes significant with the decrease of $\beta$.  Specifically,  elements of $\tilde{\A}_i$ is uniformly sampled from the interval $[-0.5+\sin(i),0.5+\sin(i)]$. Note that $\tilde{\A}_i$ and $\tilde{\A}_j$ with $i\neq j$ are drawn from different distributions. $\hat{\A}_{i,t}$ is generated according to $\y_{i,t}\in\{1,-1\}$ which is generated uniformly. When $\y_{i,t}=1$, $\hat{\A}_{i,t}$ is generated by sampling from a time-varying distribution $N((1+0.5\sin(t))\cdot\1, \I)$. When $\y_{i,t} = -1$, $\hat{\A}_{i,t}$ is generated by sampling from another time-varying distribution $N((-1+0.5\sin(t))\cdot\1, \I)$. Due to this correlation, $\y_{i,t}$ can be considered as the label of the instance $\hat{\A}_{i,t}$.

\textbf{Real Data.}
The real public datasets include \textit{SUSY}\footnote{\url{https://www.csie.ntu.edu.tw/~cjlin/libsvmtools/datasets/binary.html\#SUSY}} { ($5,000,000$ samples)}, \textit{room-occupancy}\footnote{\url{https://archive.ics.uci.edu/ml/datasets/Occupancy+Detection+}} { ($20,560$ samples)},  \textit{usenet2}\footnote{\url{http://mlkd.csd.auth.gr/concept_drift.html}} { ($1,500$ samples)}, and \textit{spam}\footnote{\url{http://mlkd.csd.auth.gr/concept_drift.html}} { ($9,324$ samples)}. \textit{SUSY} is a large-scale binary classification dataset, and we use the whole dataset to general two kinds of data: the stochastic data and the adversarial data. The stochastic data is generated by using some instances, e.g., $80\%$ of the whole dataset, and then allocating them to nodes randomly and uniformly. The adversarial data is generated by using the other instances, conducting clustering on those data to yield $n$ clusters, and then allocate every cluster to a node.  \textit{room-occupancy} is a time-series dataset, which is from a natural dynamic enviroment. Both \textit{usenet2} and \textit{spam} are  ``concept drift'' \citep{Katakis:2010:TR} datasets, for which the optimal model changes over time. For all dataset, all values of every feature have been normalized to be zero mean and one variance. We present the numerical results for the dataset \textit{SUSY} here, and place other results in supplementary materials.

\begin{figure*}[!t]
\setlength{\abovecaptionskip}{0pt}
\setlength{\belowcaptionskip}{0pt}
\centering 
\subfigure[$\beta = 0.9$]{\includegraphics[width=0.24\columnwidth]{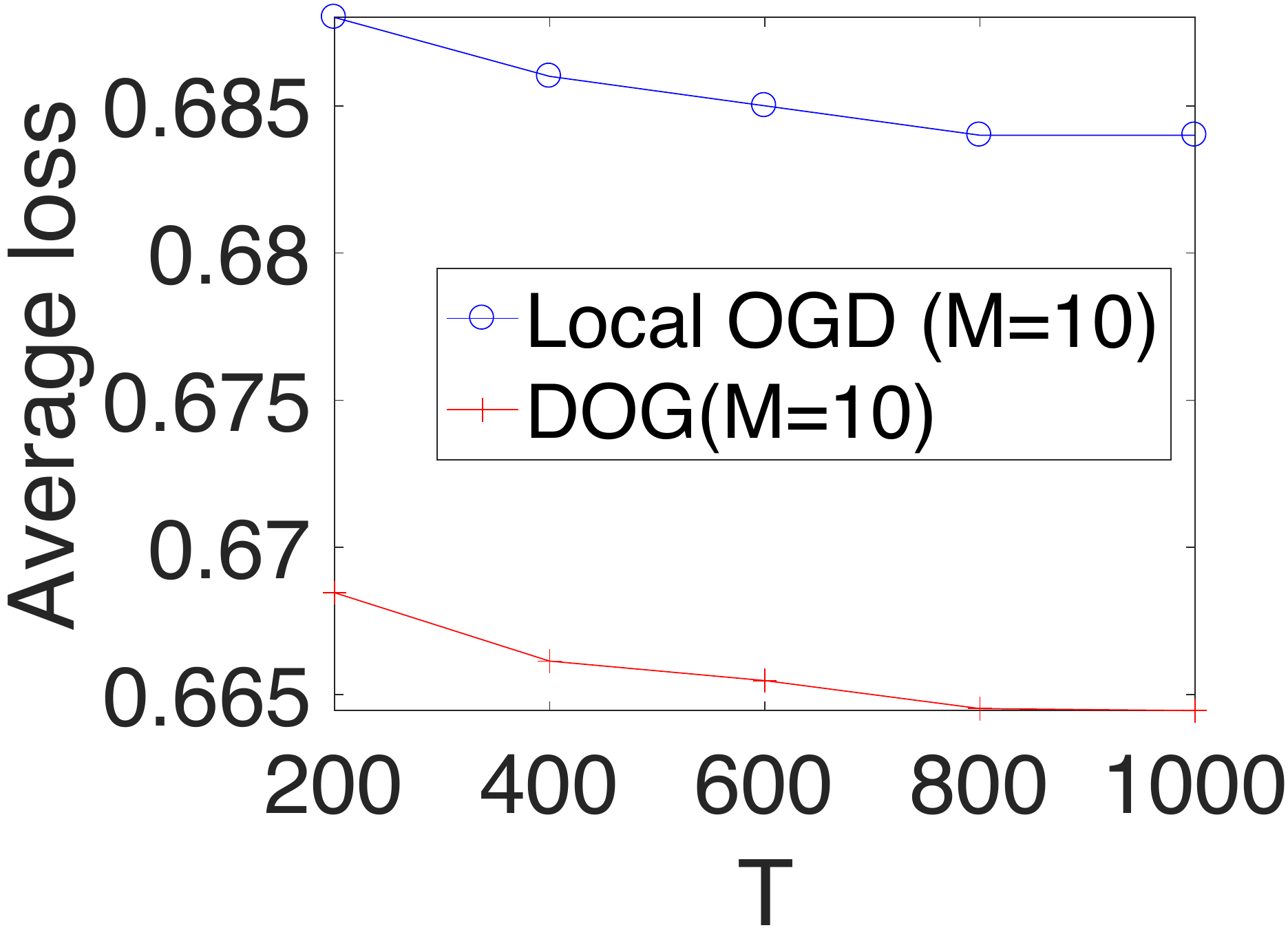}\label{figure_diff_stoc_adver_components_beta09}}
\subfigure[$\beta = 0.7$]{\includegraphics[width=0.24\columnwidth]{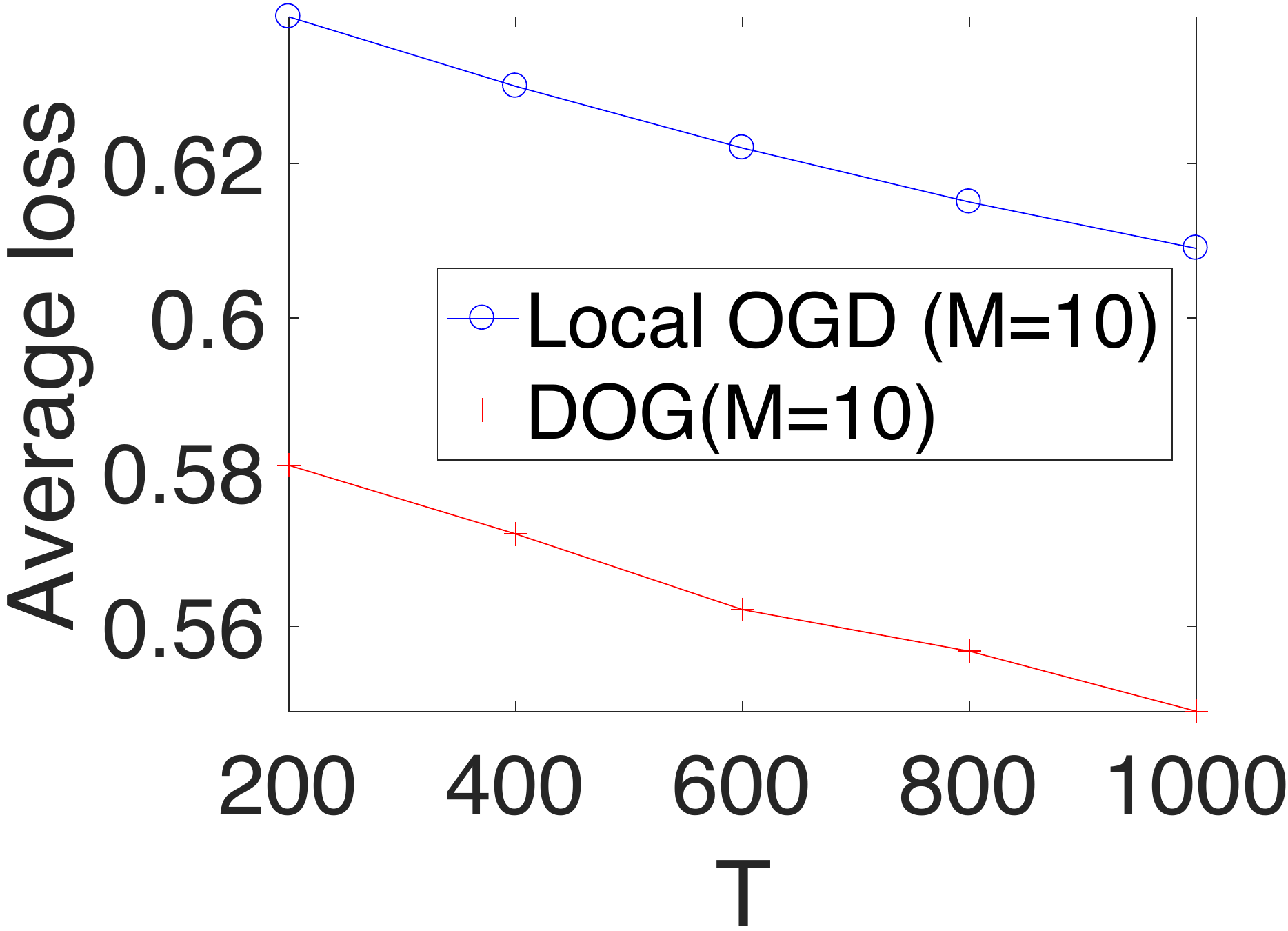}\label{figure_diff_stoc_adver_components_beta07}}
\subfigure[$\beta = 0.5$]{\includegraphics[width=0.24\columnwidth]{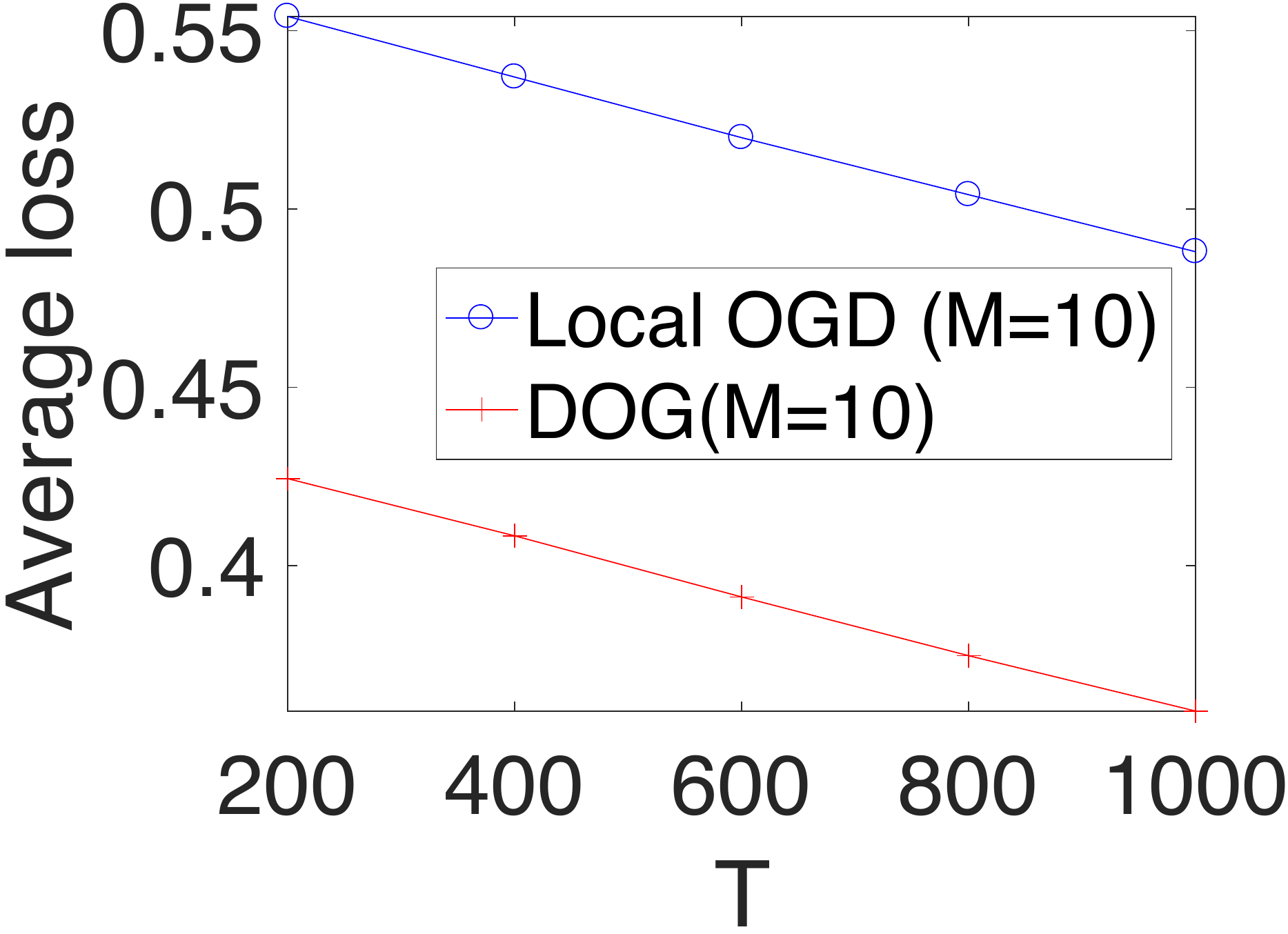}\label{figure_diff_stoc_adver_components_beta05}}
\subfigure[$\beta = 0.3$]{\includegraphics[width=0.24\columnwidth]{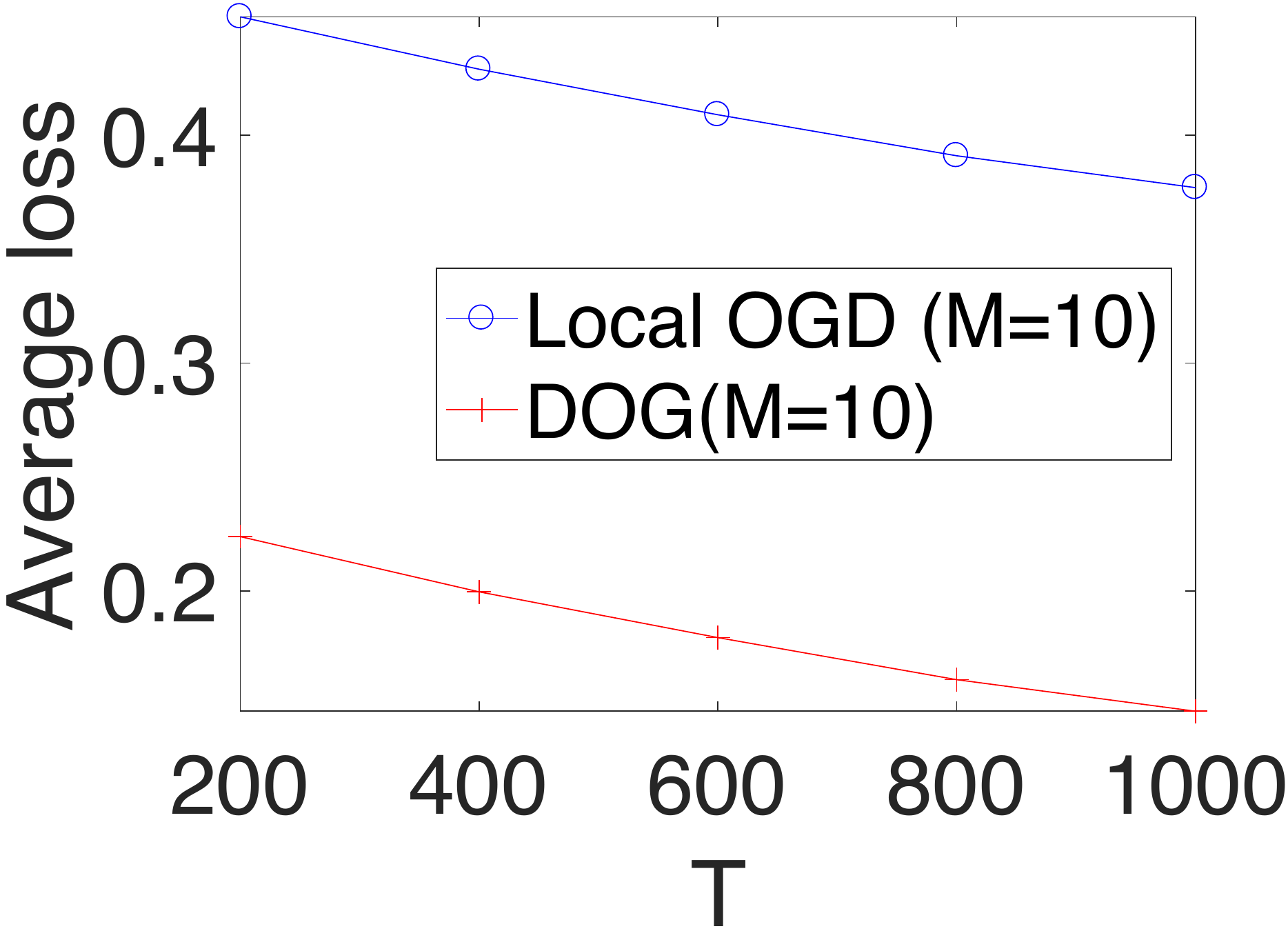}\label{figure_diff_stoc_adver_components_beta03}}
\caption{Local OGD vs. DOG on synthetic data with different ratios of the adversarial component. (10000 nodes with ring topology)}
\label{figure_reduce_stochastic_regret}
\end{figure*}

\begin{figure*}[!t]
\setlength{\abovecaptionskip}{0pt}
\setlength{\belowcaptionskip}{0pt}
\centering 
\subfigure[\textit{synthetic}, $10000$ nodes, $\beta = 0.1$, random topology]{\includegraphics[width=0.235\columnwidth]{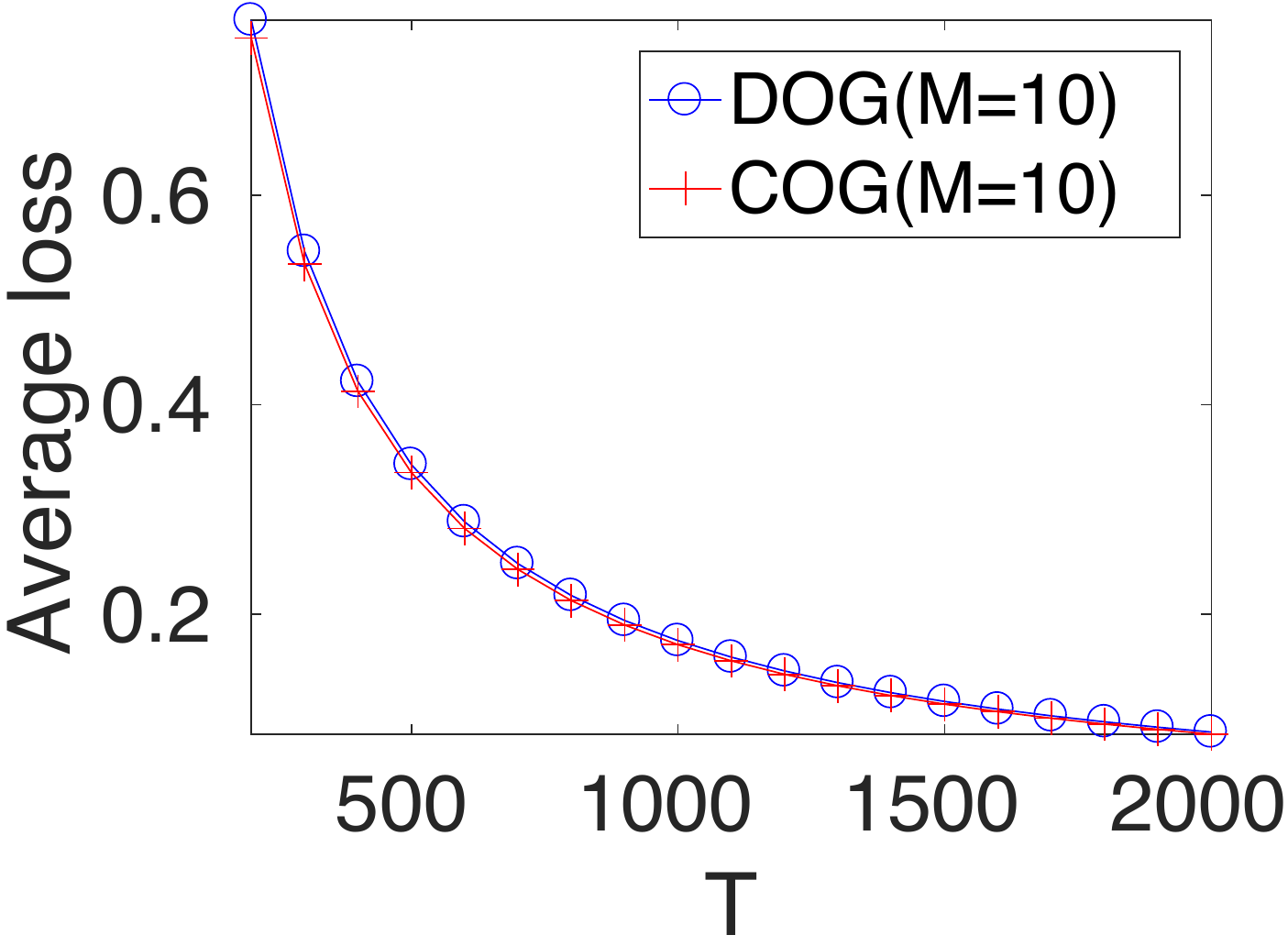}\label{figure_ave_loss_iteration}}
\subfigure[\textit{SUSY}, $100\%$ stochastic data]{\includegraphics[width=0.24\columnwidth]{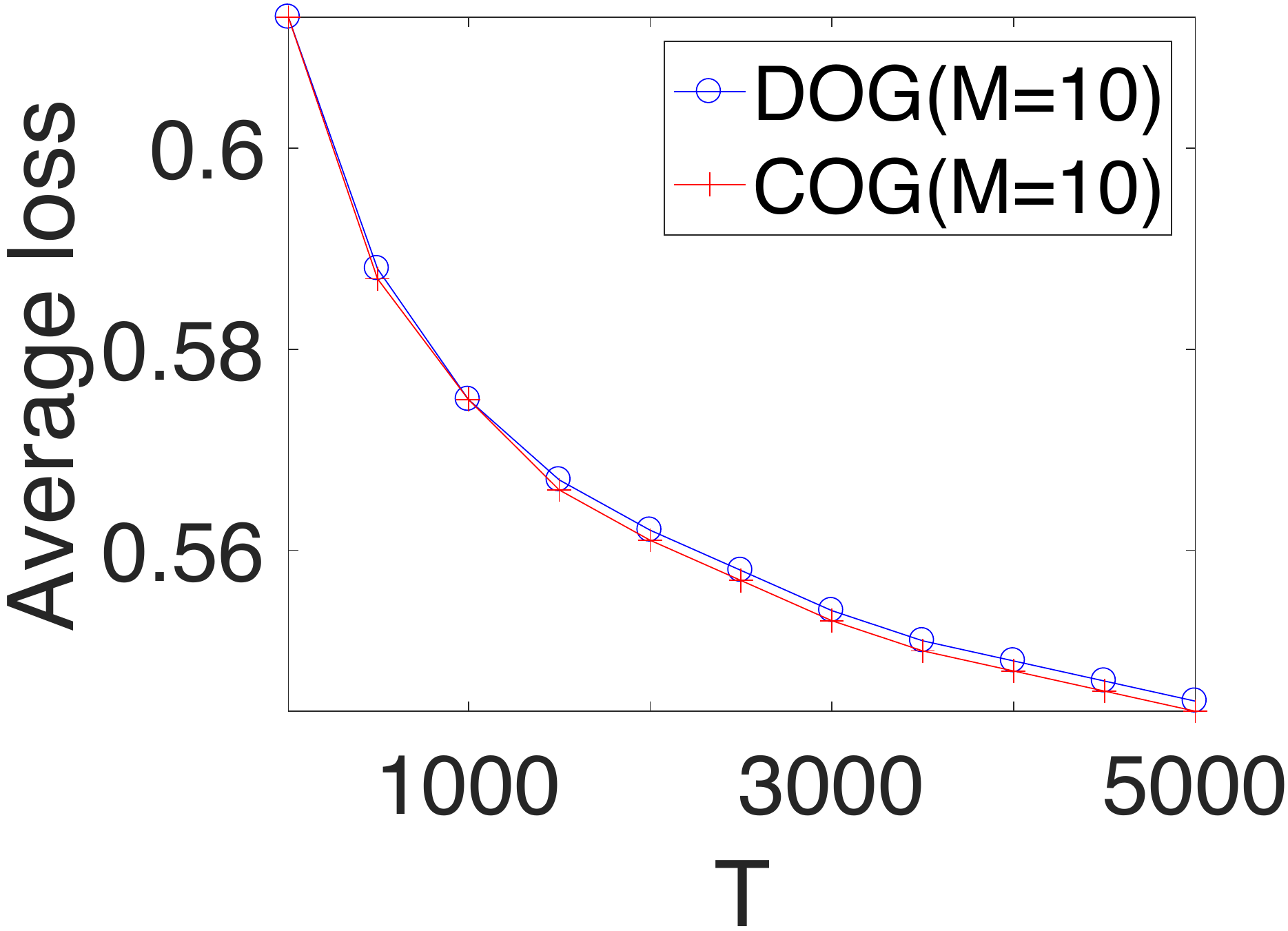}\label{figure_average_loss_susy}}
\subfigure[\textit{SUSY}, $80\%$ stochastic data]{\includegraphics[width=0.24\columnwidth]{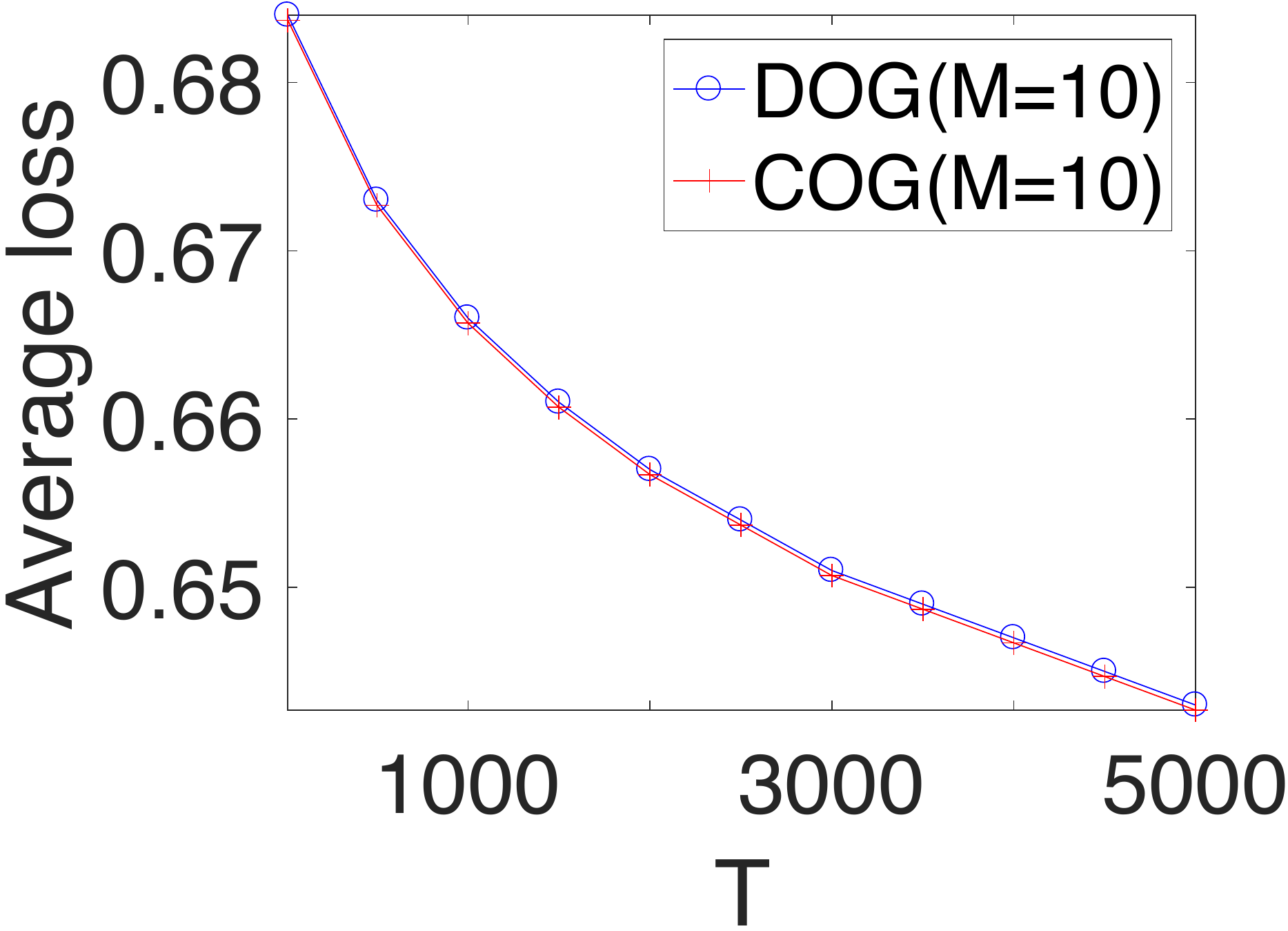}\label{figure_average_loss_susy_02}}
\subfigure[\textit{SUSY}, $50\%$ stochastic data]{\includegraphics[width=0.24\columnwidth]{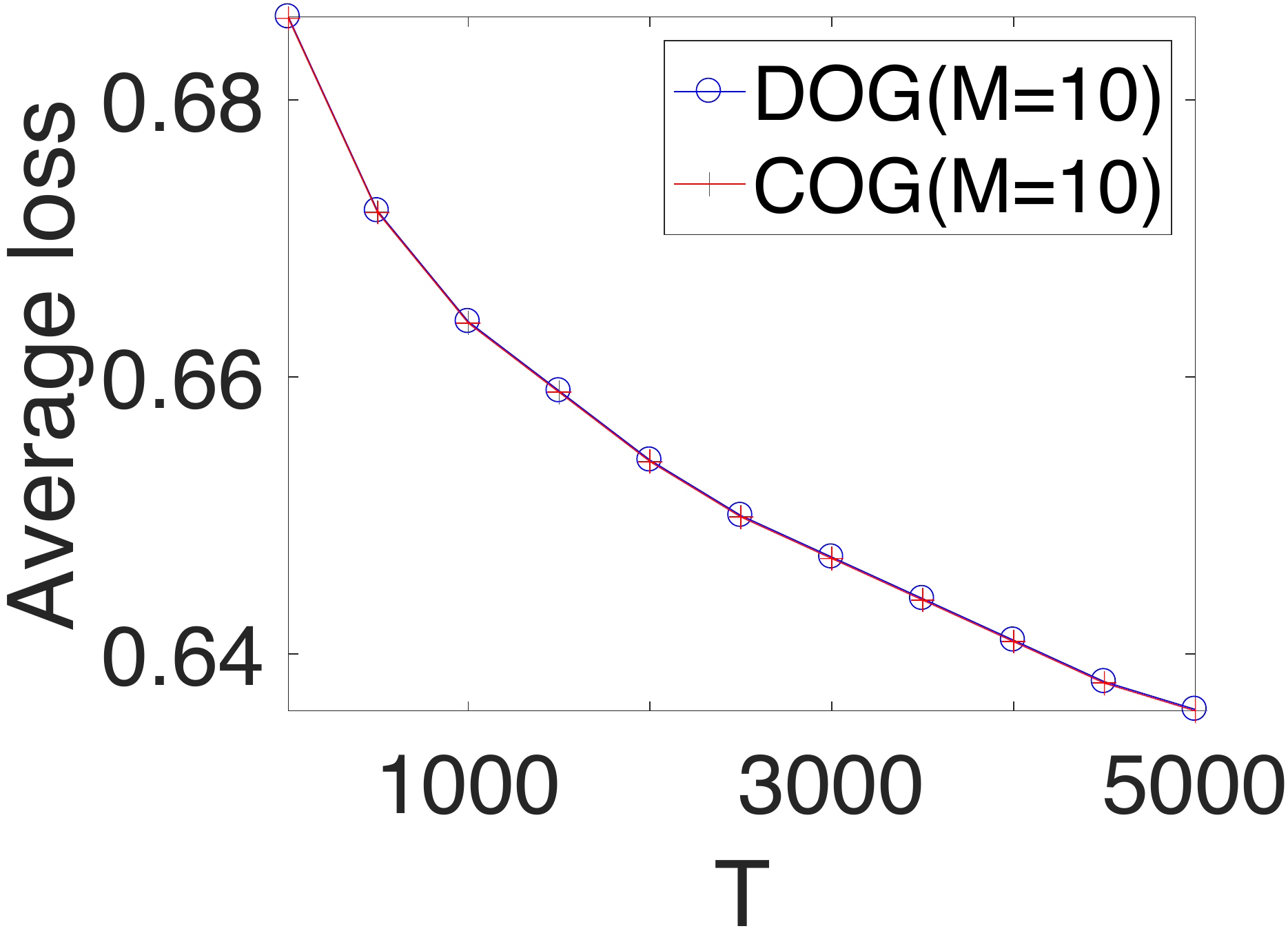}\label{figure_average_loss_susy_05}}
\caption{DOG vs. COG ($1000$ nodes with ring topolgy for real data).}
\label{figure_compare_loss}
\end{figure*}

\begin{figure*}[!t]
\setlength{\abovecaptionskip}{0pt}
\setlength{\belowcaptionskip}{0pt}
\centering 
\subfigure[\textit{synthetic}, $\beta = 0.1$, random topology]{\includegraphics[width=0.24\columnwidth]{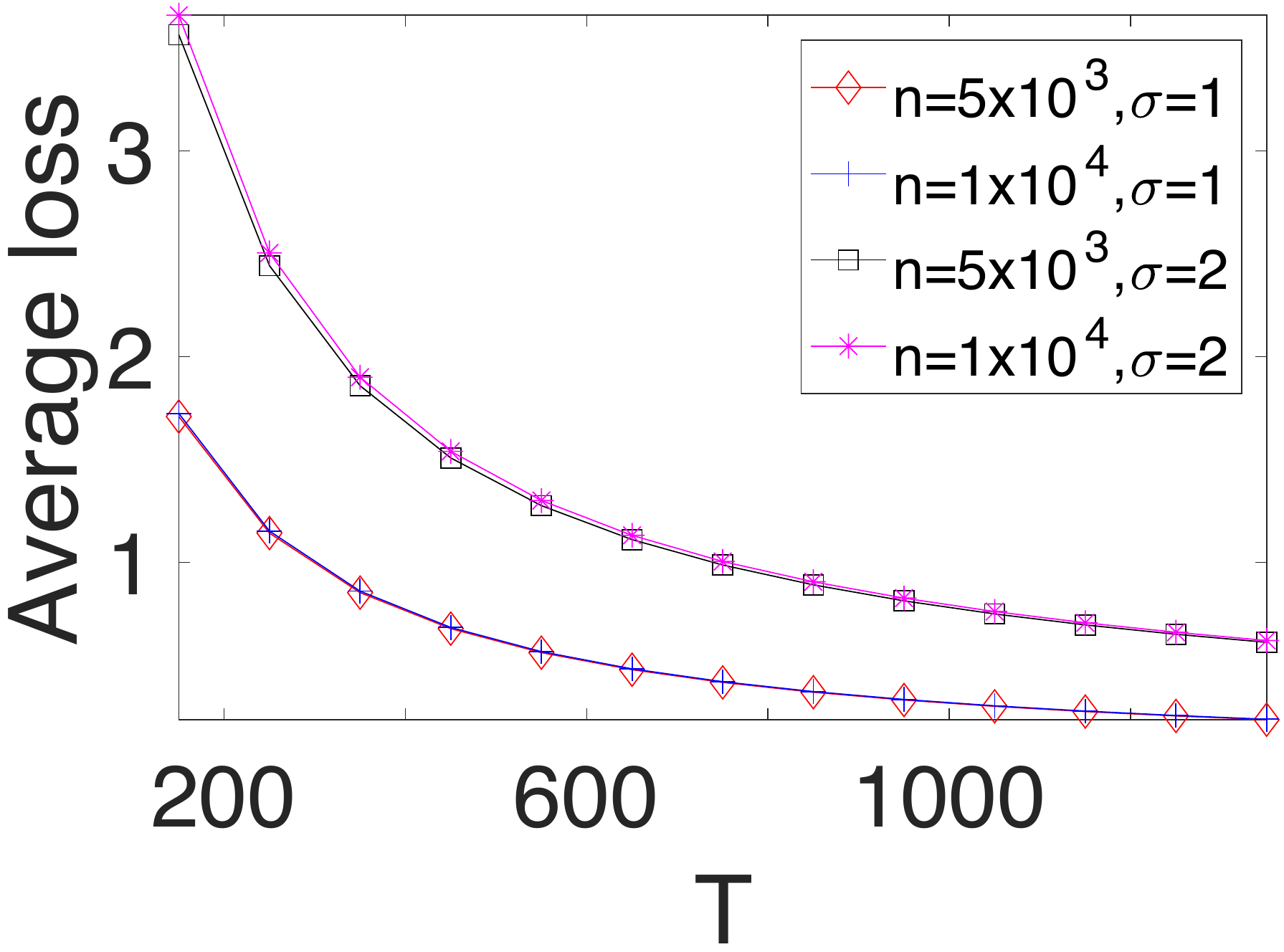}\label{figure_ave_loss_network_size_synthetic}}
\subfigure[\textit{SUSY}, $100\%$ stochastic data, ring topology]{\includegraphics[width=0.24\columnwidth]{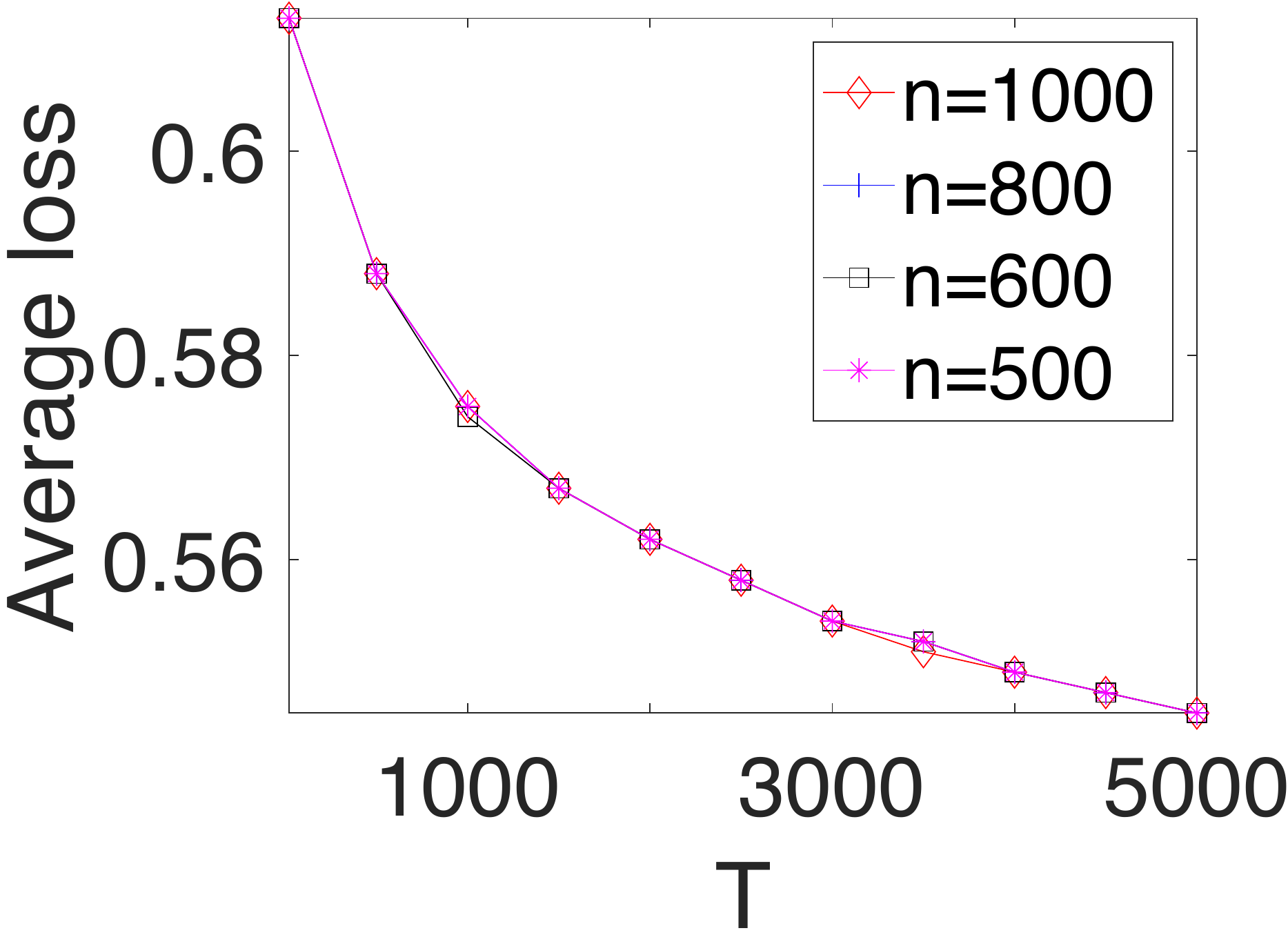}\label{figure_network_size_susy}}
\subfigure[\textit{SUSY}, $80\%$ stochastic data, ring topology]{\includegraphics[width=0.24\columnwidth]{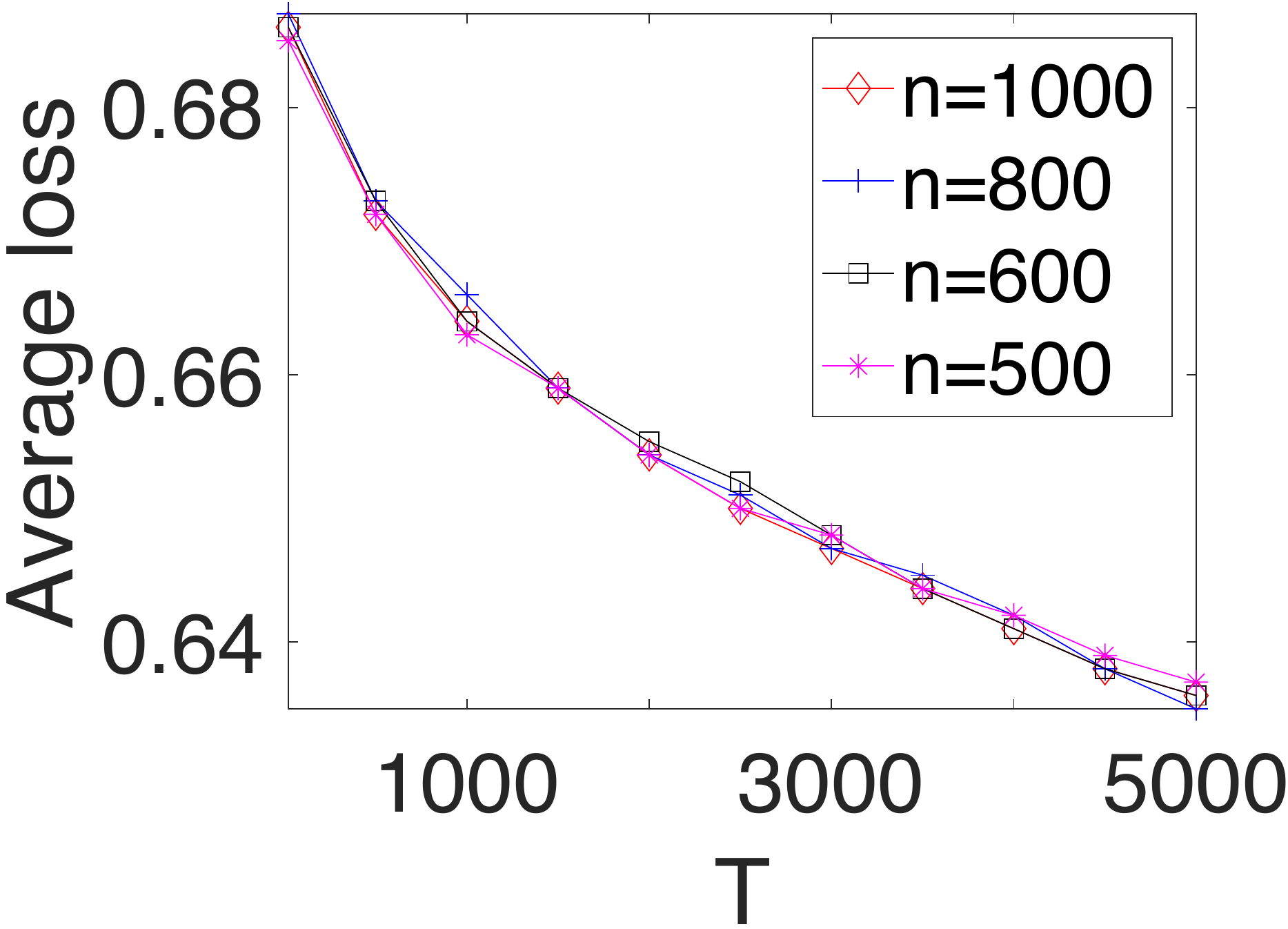}\label{figure_network_size_susy_02}}
\subfigure[\textit{SUSY}, $50\%$ stochastic data, ring topology]{\includegraphics[width=0.24\columnwidth]{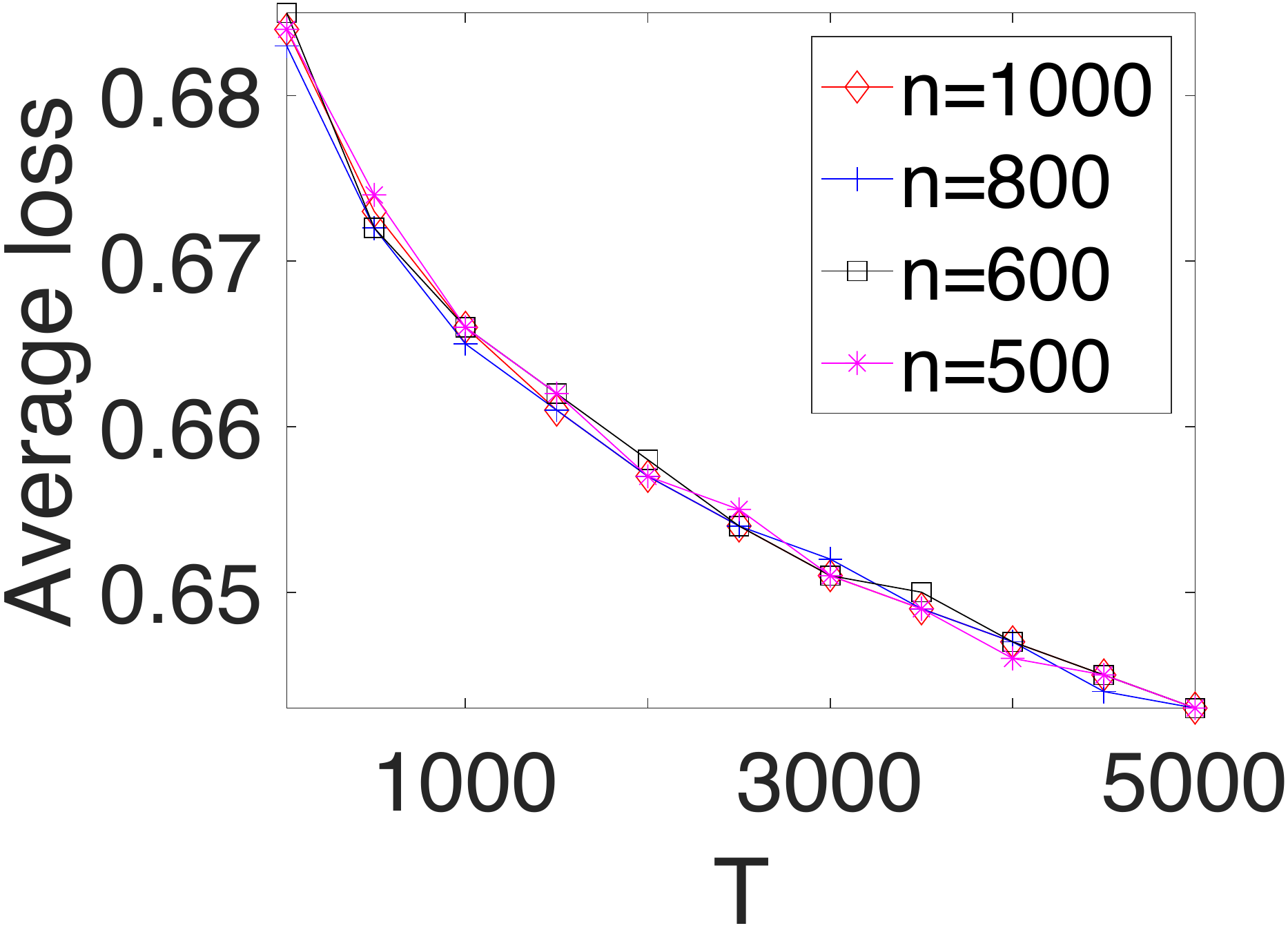}\label{figure_network_size_susy_05}}
\caption{The robustness of DOG wrt the network size.}
\label{figure_compare_network_size}
\end{figure*}

\begin{figure*}[!t]
\setlength{\abovecaptionskip}{0pt}
\setlength{\belowcaptionskip}{0pt}
\centering 
\subfigure[\textit{synthetic}, $10000$ nodes, $\beta = 0.1$]{\includegraphics[width=0.225\columnwidth]{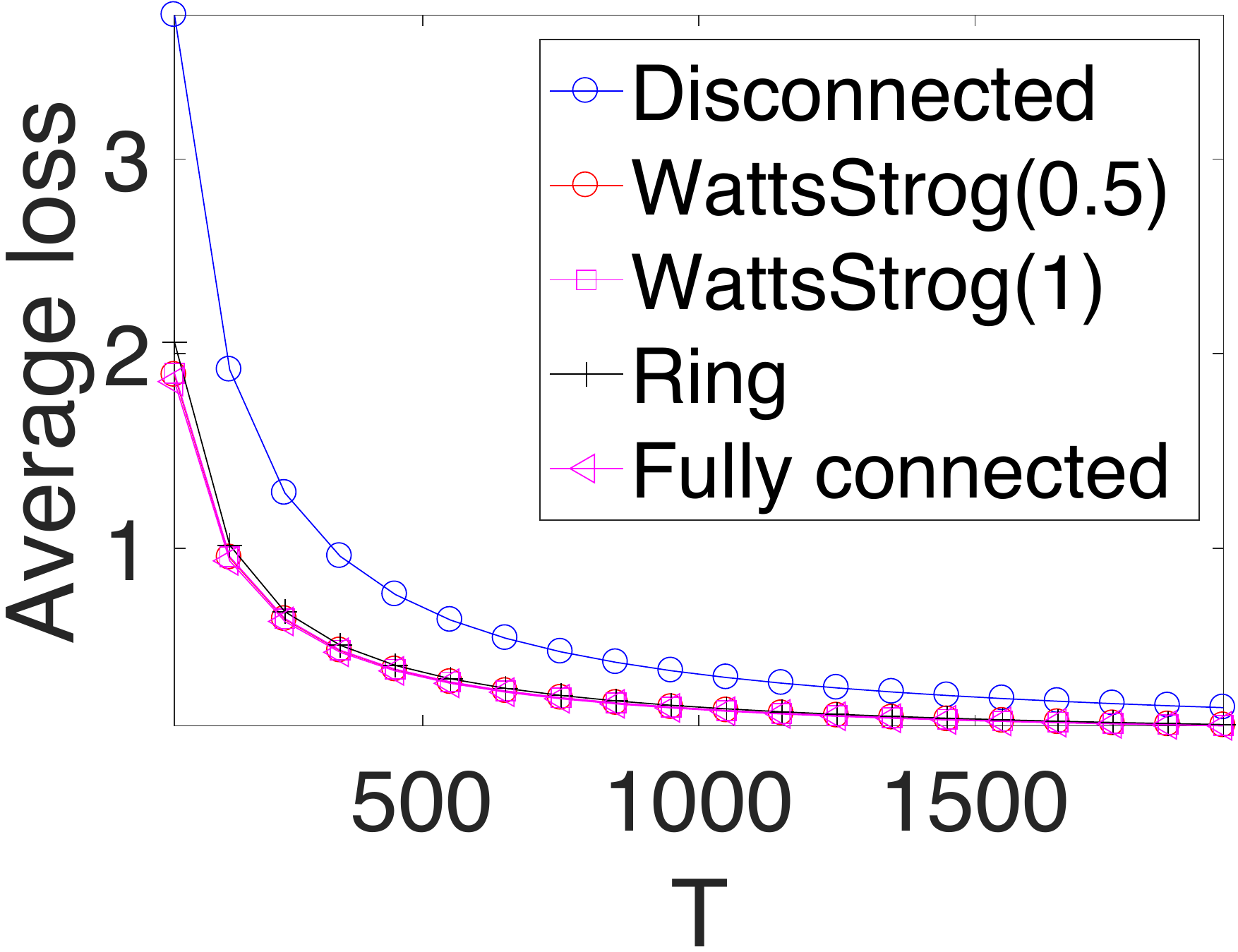}\label{figure_ave_loss_topology_synthetic}}
\subfigure[\textit{SUSY}, $1000$ nodes, $100\%$ stochastic data]{\includegraphics[width=0.24\columnwidth]{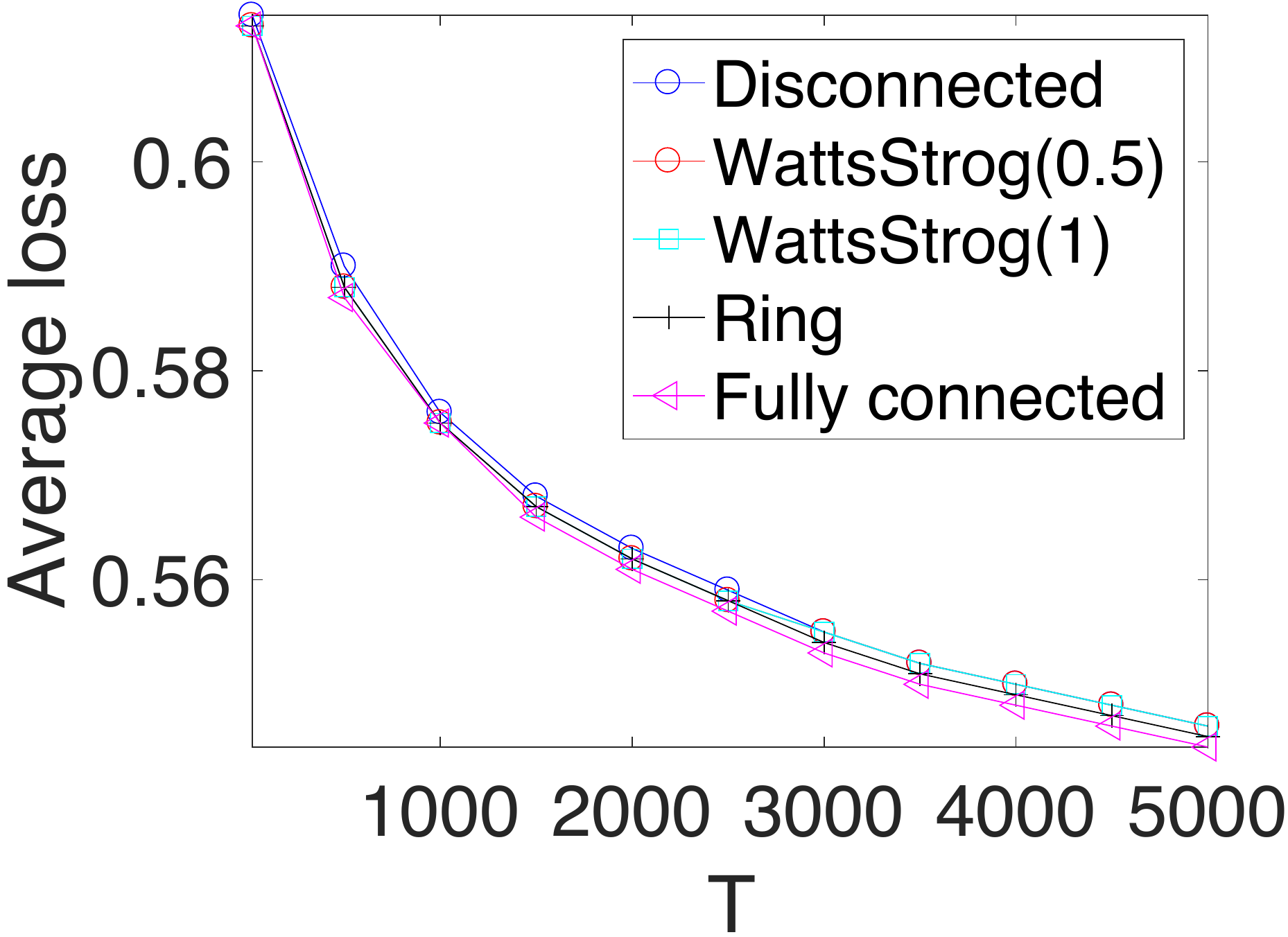}\label{figure_topology_susy}}
\subfigure[\textit{SUSY}, $1000$ nodes, $80\%$ stochastic data]{\includegraphics[width=0.24\columnwidth]{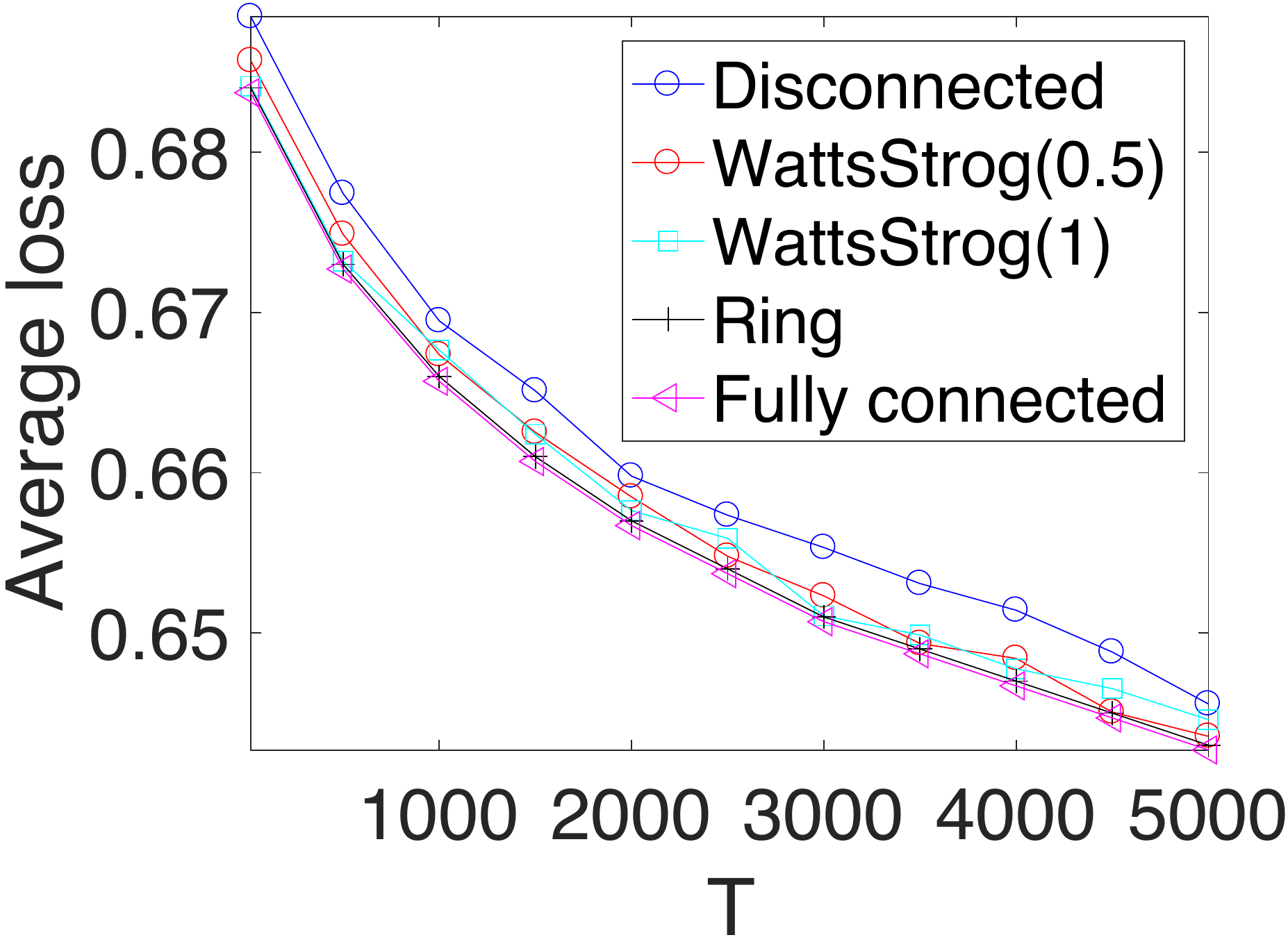}\label{figure_topology_susy_02}}
\subfigure[\textit{SUSY}, $1000$ nodes, $50\%$ stochastic data]{\includegraphics[width=0.24\columnwidth]{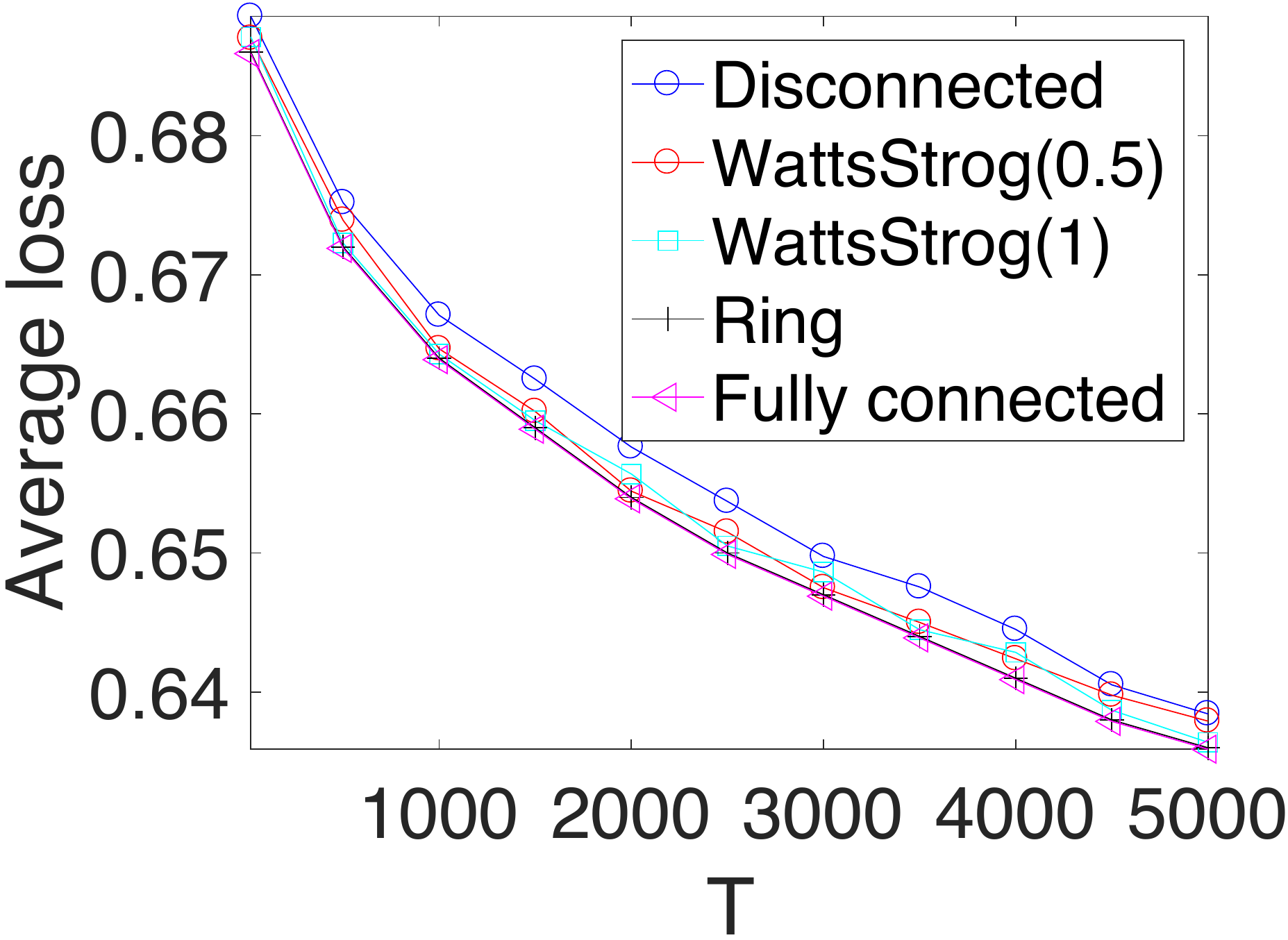}\label{figure_topology_susy_05}}
\caption{The sensitivity of DOG wrt the topology of the network.}
\label{figure_compare_topology}
\end{figure*}

\textbf{DOG is effective to reduce the stochastic component of regret.} It is compared with the local online gradient descent (Local OGD), where every node trains a local model without communication with others. We vary $\beta$ to generate different kinds of synthetic data, to obtain different balance between adversary and stochastic components. As shown before, the stochastic component of data becomes significant with the decrease of $\beta$. Figure \ref{figure_reduce_stochastic_regret} shows that DOG becomes significantly effective to reduce the stochastic component of regret for small $\beta$. It validates that exchanging models in a decentralized network is necessary and important to reduce regret, which matches with our theoretical result.  In the following empirical studies, we generate synthetic data in the setting of $\beta = 0.1$.

\textbf{DOG yields comparable performance with COG.} Figure \ref{figure_compare_loss} summarizes the performance of DOG compared with COG. 
For the synthetic dataset, we simulate a network consisting of $10000$ nodes, where every node is randomly connected with other $15$ nodes. Similarly, we simulate a ring network for the other real datasets, and every node is randomly connected with other $3$ nodes.  Under these settings, we can observe that both DOG and COG are effective for the online learning tasks on all the datasets. Besides, DOG achieves slightly worse performance than COG. It is significant when the adversarial component becomes large, which can be verified by Figures \ref{figure_average_loss_susy_02} and \ref{figure_average_loss_susy_05}.

\textbf{The performance of DOG is robust to the network size, but is sensitive to the variance of the stochastic data.} Figure~\ref{figure_compare_network_size} summarizes the effect of the network size on the performance of DOG. We change the number of nodes different datasets. The synthetic dataset is tested by using the random topology, and those real datasets are tested by using the ring topology. Figure~\ref{figure_compare_network_size} draws the curves of average loss over time steps. We observe that the average loss curves are mostly overlapped with different nodes. It shows that DOG is robust to the network size (or number of users), which validates our theory, that is, the average regret\footnote{The `average regret' equals to $\frac{\text{regret}}{nT}$. Our theoretical analysis shows that the average regret of DOG is $\Ocal{\sqrt{\frac{G^2}{T}+\frac{\sigma^2}{nT}}}$.} does not increase with the number of nodes. Furthermore, we observe that the average loss becomes large with the increase of the variance of stochastic data, which validates our theoretical result nicely.  

\textbf{The performance of DOG can be improved in a well-connected network.} Figure~\ref{figure_compare_topology} shows the effect of the topology of the network on the performance of DOG, where five different topologies are used. Besides, the ring topology, the \textit{Disconnected} topology means there are no edges in the network, and every node does not share its local model to others. The \textit{Fully connected} topology means all nodes are connected, where DOG de-generates to be COG. The topology  \textit{WattsStrogatz} represents a Watts-Strogatz small-world graph, for which we can use a parameter to control the number of stochastic edges (set as $0.5$ and $1$ in this paper). The result shows \textit{Fully connected} enjoys the best performance, because that $\rho = 0$ for it, and \textit{Disconnected} suffers the worst performance due to $\rho = 1$ for it. Other topologies owns $0<\rho<1$ for them.

\section{Conclusion}
We investigate decentralized online learning problem, where the loss is incurred by both adversary and stochastic components.  We define a new dynamic regret, and analyze a decentralized online gradient method theoretically. It shows that the communication is only effective to decrease the regret caused by the stochastic component, and thus users can benefit from sharing their private models, instead of private data.

\balance

\bibliography{reference}

\bibliographystyle{abbrvnat}

\newpage
\onecolumn

\section*{Supplementary materials for theoretical analysis}

For math brevity, we denote the function $F_{i,t}$ by $F_{i,t}(\cdot) := \EE_{\xi_{i,t} \sim D_{i,t}} f_{i,t}(\cdot; \xi_{i,t})$ throughout proofs.

\textbf{Proof to Theorem \ref{theorem_regret_upper_bound}:}
\begin{proof}
From the regret definition, we have
\begin{align}
\nonumber
& \EE_{ \Xi_{n,t} \sim \Dcal_{n,t} } \frac{1}{n}\sum_{i=1}^n \lrincir{ f_{i,t}(\x_{i,t};\xi_{i,t}) - f_{i,t}(\x_t^\ast;\xi_{i,t}) } \le  \EE_{ \Xi_{n,t} \sim \Dcal_{n,t} } \frac{1}{n}\sum_{i=1}^n \lrangle{ \nabla f_{i,t}(\x_{i,t};\xi_{i,t}),  \x_{i,t} - \x_t^\ast } \\ \nonumber
= & \underbrace{ \EE_{ \Xi_{n,t} \sim \Dcal_{n,t} } \frac{1}{n}\sum_{i=1}^n   \lrincir{\lrangle{\nabla  f_{i,t}(\x_{i,t};\xi_{i,t}), \x_{i,t} - \bar{\x}_t } + \lrangle{\nabla  f_{i,t}(\x_{i,t};\xi_{i,t}), \bar{\x}_t - \bar{\x}_{t+1} } } }_{I_1(t)} \\ \nonumber 
&+ \underbrace{ \EE_{ \Xi_{n,t} \sim \Dcal_{n,t} } \lrangle{\frac{1}{n}\sum_{i=1}^n\nabla f_{i,t}(\x_{i,t};\xi_{i,t}), \bar{\x}_{t+1} - \x_t^\ast } }_{I_2(t)}. \nonumber
\end{align}

Now, we begin to bound $I_1(t)$.
\begin{align}
\nonumber
I_1(t) = & \underbrace{ \EE_{ \Xi_{n,t} \sim \Dcal_{n,t} }\frac{1}{n}\sum_{i=1}^n\lrangle{\nabla f_{i,t}(\x_{i,t}; \xi_{i,t}), \x_{i,t} - \bar{\x}_t } }_{J_1(t)} +  \underbrace{ \EE_{ \Xi_{n,t} \sim \Dcal_{n,t} }\lrangle{\frac{1}{n}\sum_{i=1}^n \nabla f_{i,t}(\x_{i,t};\xi_{i,t}), \bar{\x}_t - \bar{\x}_{t+1} }}_{J_2(t)}.
\end{align} For $J_1(t)$, we have
\begin{align} 
\nonumber
& J_1(t) \\ \nonumber 
= & \frac{1}{n}\EE_{ \Xi_{n,t} \sim \Dcal_{n,t} }\sum_{i=1}^n\lrangle{\nabla f_{i,t}(\x_{i,t}; \xi_{i,t}), \x_{i,t} - \bar{\x}_t } \\ \nonumber
= & \frac{1}{n}\EE_{ \Xi_{n,t-1} \sim \Dcal_{n,t-1} }\sum_{i=1}^n\lrangle{\nabla F_{i,t}(\x_{i,t}), \x_{i,t} - \bar{\x}_t } \\ \nonumber
= & \frac{1}{n}\EE_{ \Xi_{n,t-1} \sim \Dcal_{n,t-1} }\sum_{i=1}^n\lrangle{\nabla F_{i,t}(\x_{i,t}) - \nabla F_{i,t}(\bar{\x}_t), \x_{i,t} - \bar{\x}_t } + \frac{1}{n}\EE_{ \Xi_{n,t-1} \sim \Dcal_{n,t-1} }\sum_{i=1}^n\lrangle{\nabla F_{i,t}(\bar{\x}_t), \x_{i,t} - \bar{\x}_t } \\ \label{equa_J1_temp}
= & \frac{L}{n}\EE_{ \Xi_{n,t-1} \sim \Dcal_{n,t-1} }\sum_{i=1}^n \lrnorm{\x_{i,t} - \bar{\x}_t }^2 + \frac{1}{n}\EE_{ \Xi_{n,t-1} \sim \Dcal_{n,t-1} }\sum_{i=1}^n\lrangle{\nabla F_{i,t}(\bar{\x}_t), \x_{i,t} - \bar{\x}_t }. 
\end{align}  Consider the last term, and we have 
\begin{align}
\nonumber
& \frac{1}{n}\EE_{ \Xi_{n,t-1} \sim \Dcal_{n,t-1} }\sum_{i=1}^n\lrangle{\nabla F_{i,t}(\bar{\x}_t), \x_{i,t} - \bar{\x}_t } \\ \nonumber  
\refabovecir{=}{\textcircled{1}} & \frac{1}{n} \EE_{ \Xi_{n,t-1} \sim \Dcal_{n,t-1} } \tr\lrincir{ \nabla F_t(\bar{\X}_t)\Tr \lrincir{\X_t -  \bar{\X}_t} } \\ \nonumber
= & \frac{1}{n} \EE_{ \Xi_{n,t-1} \sim \Dcal_{n,t-1} } \tr\lrincir{ \nabla F_t(\bar{\X}_t)\Tr \lrincir{\sum_{s=1}^{t-1} \eta  \G_s \W_s^{t-1-s} -  \frac{1}{n}\sum_{s=1}^{t-1} \eta  \G_s \W_s^{t-1-s} \v_1 \v_1\Tr} } \\ \nonumber
= & \frac{1}{n} \tr\lrincir{ \nabla F_t(\bar{\X}_t)\Tr \lrincir{\sum_{s=1}^{t-1} \eta  \nabla F_s(\X_s) \W_s^{t-1-s} -  \frac{1}{n}\sum_{s=1}^{t-1} \eta  \nabla F_s(\X_s) \W_s^{t-1-s} \v_1 \v_1\Tr} } \\ \nonumber
= & \frac{\eta}{n} \sum_{s=1}^{t-1} \tr\lrincir{ \nabla F_t(\bar{\X}_t)\Tr   \nabla F_s(\X_s)  \W_s^{t-1-s} \lrincir{ \I_n -  \frac{1}{n} \v_1 \v_1\Tr} } \\ \nonumber
= & \frac{\eta}{n} \sum_{s=1}^{t-1} \tr\lrincir{ \nabla F_t(\bar{\X}_t)\Tr   \nabla F_s(\X_s)   \lrincir{ \W_s^{t-1-s} -  \frac{1}{n} \v_1 \v_1\Tr} } \\ \nonumber
\le & \frac{\eta}{2n} \sum_{s=1}^{t-1} \lrincir{ \rho^{t-1-s}\lrnorm{ \nabla F_t(\bar{\X}_t)}_F^2 + \frac{1}{\rho^{t-1-s}} \lrnorm{ \nabla F_s(\X_s)   \lrincir{ \W_s^{t-1-s} -  \frac{1}{n} \v_1 \v_1\Tr} }_F^2 } \\ \nonumber
\refabovecir{\le}{\textcircled{2}} & \frac{\eta}{2n} \sum_{s=1}^{t-1} \lrincir{ \rho^{t-1-s} \lrnorm{ \nabla F_t(\bar{\X}_t)}_F^2 +  \rho^{t-1-s}\lrnorm{ \nabla F_s(\X_s) }_F^2   } \\ \nonumber
\le & \frac{\eta}{2n} \sum_{s=1}^{t-1} \rho^{t-1-s} \lrincir{ \lrnorm{\nabla F_t(\bar{\X}_t)}_F^2 + \lrnorm{\nabla F_s(\X_s)}_F^2  } \\ \nonumber
\le & \eta G^2 \sum_{s=1}^{t-1} \rho^{t-1-s} \\ \nonumber
\le & \frac{\eta G^2}{1-\rho}. 
\end{align} $\textcircled{1}$ holds due to $\X_t = [\x_{1,t}; \x_{2,t}; \cdots; \x_{n,t}]$, and $\bar{\X}_t = [\bar{\x}_t; \bar{\x}_t; \cdots; \bar{\x}_t]$. $\textcircled{2}$ holds due to Lemma \ref{Lemma_hanlin_1}.

Substitute it into \eqref{equa_J1_temp}, and we thus have
\begin{align}
\nonumber
J_1(t) = \frac{L}{n}\EE_{ \Xi_{n,t-1} \sim \Dcal_{n,t-1} }\sum_{i=1}^n \lrnorm{\x_{i,t} - \bar{\x}_t }^2 + \frac{\eta G^2}{1-\rho}. 
\end{align}

For $J_2(t)$, we have
\begin{align}
\nonumber
& J_2(t) \\ \nonumber 
= & \EE_{ \Xi_{n,t} \sim \Dcal_{n,t} }\lrangle{\frac{1}{n}\sum_{i=1}^n\nabla f_{i,t}(\x_{i,t};\xi_{i,t}), \bar{\x}_t - \bar{\x}_{t+1} } \\ \nonumber
\le & \frac{\eta}{2}\EE_{ \Xi_{n,t} \sim \Dcal_{n,t} } \lrnorm{\frac{1}{n}\sum_{i=1}^n \nabla f_{i,t}(\x_{i,t};\xi_{i,t})}^2 + \frac{1}{2\eta} \EE_{ \Xi_{n,t} \sim \Dcal_{n,t} }\lrnorm{ \bar{\x}_t - \bar{\x}_{t+1}}^2  \\ \nonumber
\le & \frac{\eta}{2}\EE_{ \Xi_{n,t} \sim \Dcal_{n,t} }\lrnorm{\frac{1}{n}\sum_{i=1}^n \lrincir{\nabla  f_{i,t}(\x_{i,t};\xi_{i,t}) - \nabla F_{i,t}(\x_{i,t}) + \nabla F_{i,t}(\x_{i,t})} }^2 + \frac{1}{2\eta} \EE_{ \Xi_{n,t} \sim \Dcal_{n,t} }\lrnorm{ \bar{\x}_t - \bar{\x}_{t+1}}^2  \\ \nonumber
\le &  \eta\EE_{ \Xi_{n,t} \sim \Dcal_{n,t} }\lrnorm{\frac{1}{n}\sum_{i=1}^n \lrincir{ \nabla f_{i,t}(\x_{i,t};\xi_{i,t}) - \nabla F_{i,t}(\x_{i,t}) } }^2 + \eta \EE_{ \Xi_{n,t} \sim \Dcal_{n,t} }\lrnorm{\frac{1}{n}\sum_{i=1}^n\nabla F_{i,t}(\x_{i,t})}^2  + \frac{1}{2\eta} \EE_{ \Xi_{n,t} \sim \Dcal_{n,t} }\lrnorm{ \bar{\x}_t - \bar{\x}_{t+1}}^2  \\ \nonumber
\refabovecir{\le}{\textcircled{1}} & \frac{\eta}{n} \sigma^2 + \eta \EE_{ \Xi_{n,t-1} \sim \Dcal_{n,t-1} }\lrnorm{ \frac{1}{n}\sum_{i=1}^n \lrincir{ \nabla F_{i,t}(\x_{i,t}) - \nabla F_{i,t}(\bar{\x}_t) + \nabla F_{i,t}(\bar{\x}_t) } }^2  + \frac{1}{2\eta} \EE_{ \Xi_{n,t} \sim \Dcal_{n,t} }\lrnorm{ \bar{\x}_t - \bar{\x}_{t+1}}^2 \\ \nonumber
\le & \frac{\eta}{n} \sigma^2 + 2\eta \EE_{ \Xi_{n,t-1} \sim \Dcal_{n,t-1} }\lrnorm{\frac{1}{n}\sum_{i=1}^n \lrincir{ \nabla F_{i,t}(\x_{i,t}) - \nabla F_{i,t}(\bar{\x}_t) } }^2  \\ \nonumber 
&+ 2\eta \EE_{ \Xi_{n,t-1} \sim \Dcal_{n,t-1} }\lrnorm{\frac{1}{n}\sum_{i=1}^n\nabla F_{i,t}(\bar{\x}_t)}^2 + \frac{1}{2\eta} \EE_{ \Xi_{n,t} \sim \Dcal_{n,t} }\lrnorm{ \bar{\x}_t - \bar{\x}_{t+1}}^2 \\ \nonumber
\le & \frac{\eta}{n} \sigma^2 + \frac{2\eta}{n} \EE_{ \Xi_{n,t-1} \sim \Dcal_{n,t-1} }\sum_{i=1}^n\lrnorm{ \nabla F_{i,t}(\x_{i,t}) - \nabla F_{i,t}(\bar{\x}_t)  }^2 + 2\eta G^2 + \frac{1}{2\eta} \EE_{ \Xi_{n,t} \sim \Dcal_{n,t} }\lrnorm{ \bar{\x}_t - \bar{\x}_{t+1}}^2 \\ \nonumber
\refabovecir{\le}{\textcircled{2}} & \frac{\eta}{n} \sigma^2 + \frac{2\eta L^2}{n}\EE_{ \Xi_{n,t-1} \sim \Dcal_{n,t-1} }\sum_{i=1}^n \lrnorm{\x_{i,t} - \bar{\x}_t }^2 + 2\eta G^2 + \frac{1}{2\eta} \EE_{ \Xi_{n,t} \sim \Dcal_{n,t} }\lrnorm{ \bar{\x}_t - \bar{\x}_{t+1}}^2.
\end{align} $\textcircled{1}$ holds due to
\begin{align}
\nonumber
& \EE_{ \Xi_{n,t} \sim \Dcal_{n,t} }\lrnorm{\frac{1}{n}\sum_{i=1}^n \lrincir{ \nabla f_{i,t}(\x_{i,t};\xi_{i,t}) - \nabla F_{i,t}(\x_{i,t}) } }^2 \\ \nonumber
= & \frac{1}{n^2}\EE_{ \Xi_{n,t} \sim \Dcal_{n,t} }\lrincir{ \sum_{i=1}^n \EE_{ \xi_{i,t} \sim D_{i,t} }\lrnorm{ \nabla f_{i,t}(\x_{i,t};\xi_{i,t}) - \nabla F_{i,t}(\x_{i,t}) }^2  } \\ \nonumber 
+ & \frac{2}{n^2}\EE_{ \Xi_{n,t} \sim \Dcal_{n,t} }\sum_{i=1}^n\sum_{j=1, j\neq i}^n\lrangle{  \nabla f_{i,t}(\x_{i,t};\xi_{i,t}) - \nabla F_{i,t}(\x_{i,t}),   \nabla f_{j,t}(\x_{j,t};\xi_{j,t}) - \nabla F_{j,t}(\x_{j,t})} \\ \nonumber
= & \frac{1}{n^2}\EE_{ \Xi_{n,t-1} \sim \Dcal_{n,t-1} }\sum_{i=1}^n \EE_{ \xi_{i,t} \sim D_{i,t} }\lrnorm{ \nabla f_{i,t}(\x_{i,t};\xi_{i,t}) - \nabla F_{i,t}(\x_{i,t}) }^2 + 0 \\ \nonumber
\le & \frac{1}{n} \sigma^2.
\end{align} $\textcircled{2}$ holds due to $F_{i,t}$ has $L$ Lipschitz gradients.

 Therefore, we obtain
\begin{align}
\nonumber
& I_1(t) =  (J_1(t) + J_2(t)) \\ \nonumber
= &   \lrincir{ \frac{L}{n} \EE_{ \Xi_{n,t-1} \sim \Dcal_{n,t-1} }\sum_{i=1}^n \lrnorm{\x_{i,t} - \bar{\x}_t}^2 + \frac{\eta}{n} \sigma^2 + \frac{2\eta L^2}{n}\EE_{ \Xi_{n,t-1} \sim \Dcal_{n,t-1} }\sum_{i=1}^n \lrnorm{\x_{i,t} - \bar{\x}_t }^2 } \\ \nonumber
& +  \lrincir{ \lrincir{2+\frac{1}{1-\rho}}\eta G^2 + \frac{1}{2\eta} \EE_{ \Xi_{n,t} \sim \Dcal_{n,t} }\lrnorm{ \bar{\x}_t - \bar{\x}_{t+1}}^2 } \\ \nonumber
\le &   \lrincir{ \frac{L}{n} + \frac{2\eta L^2}{n} }\EE_{ \Xi_{n,t-1} \sim \Dcal_{n,t-1} }\sum_{i=1}^n\lrnorm{\x_{i,t} - \bar{\x}_t }^2   + \lrincir{2+\frac{1}{1-\rho}}\eta G^2 + \frac{\eta  \sigma^2}{n} +  \frac{1}{2\eta} \EE_{ \Xi_{n,t} \sim \Dcal_{n,t} }\lrnorm{ \bar{\x}_t - \bar{\x}_{t+1}}^2.
\end{align}

Therefore, we have 
\begin{align}
\nonumber
& \sum_{t=1}^T I_1(t) \\ \nonumber 
\le &  \lrincir{ \frac{L}{n} + \frac{2\eta L^2}{n} }\EE_{ \Xi_{n,t-1} \sim \Dcal_{n,t-1} }\sum_{i=1}^n\sum_{t=1}^T\lrnorm{\x_{i,t} - \bar{\x}_t }^2   + \lrincir{2+\frac{1}{1-\rho}}T\eta G^2  + \frac{T\eta  \sigma^2}{n}  +  \frac{1}{2\eta} \EE_{ \Xi_{n,t} \sim \Dcal_{n,t} }\sum_{t=1}^T\lrnorm{ \bar{\x}_t - \bar{\x}_{t+1}}^2.
\end{align}

Now, we begin to bound $I_2(t)$. Denote that the update rule is 
\begin{align}
\nonumber
\x_{i,t+1} = \sum_{j=1}^n \W_{ij}\x_{j,t} - \eta \nabla f_{i,t}(\x_{i,t};\xi_{i,t}).
\end{align}  According to Lemma \ref{Lemma_average_update_rule}, we have 
\begin{align}
\label{equa_thoerem_update_rule_equivalent}
\bar{\x}_{t+1} = \bar{\x}_t - \eta \lrincir{\frac{1}{n}\sum_{i=1}^n \nabla f_{i,t}(\x_{i,t};\xi_{i,t})}.
\end{align} 
Denote a new auxiliary function $\phi(\z)$ as 
\begin{align}
\nonumber
\phi(\z) = \lrangle{\frac{1}{n}\sum_{i=1}^n \nabla f_{i,t}(\x_{i,t};\xi_{i,t}), \z} + \frac{1}{2\eta}\lrnorm{\z - \bar{\x}_t}^2.
\end{align} 

It is trivial to verify that \eqref{equa_thoerem_update_rule_equivalent} satisfies the first-order optimality condition of the optimization problem: $\min_{\z\in\RR^d} \phi(\z)$, that is,
\begin{align}
\nonumber
\nabla \phi(\bar{\x}_{t+1}) = \0.
\end{align} We thus have 
\begin{align}
\nonumber
\bar{\x}_{t+1} = \argmin_{\z\in\RR^d} \phi(\z) = \argmin_{\z\in\RR^d} \lrangle{\frac{1}{n}\sum_{i=1}^n \nabla f_{i,t}(\x_{i,t};\xi_{i,t}), \z} + \frac{1}{2\eta}\lrnorm{\z - \bar{\x}_t}^2.
\end{align} Furthermore, denote a new auxiliary variable $\bar{\x}_{\tau}$ as  
\begin{align}
\nonumber
\bar{\x}_{\tau} = \bar{\x}_{t+1} + \tau \lrincir{\x_t^\ast - \bar{\x}_{t+1}},
\end{align} where $0< \tau \le 1$. According to the optimality of $\bar{\x}_{t+1}$, we have
\begin{align}
\nonumber
& 0 \le \phi(\bar{\x}_{\tau}) - \phi(\bar{\x}_{t+1}) \\ \nonumber
= & \lrangle{\frac{1}{n}\sum_{i=1}^n \nabla f_{i,t}(\x_{i,t};\xi_{i,t}), \bar{\x}_{\tau} - \bar{\x}_{t+1}} + \frac{1}{2\eta}\lrincir{ \lrnorm{\bar{\x}_{\tau} - \bar{\x}_t}^2 - \lrnorm{\bar{\x}_{t+1} - \bar{\x}_t}^2 } \\ \nonumber
= & \lrangle{\frac{1}{n}\sum_{i=1}^n \nabla f_{i,t}(\x_{i,t};\xi_{i,t}), \tau \lrincir{\x_t^\ast - \bar{\x}_{t+1}}} + \frac{1}{2\eta}\lrincir{ \lrnorm{\bar{\x}_{t+1} + \tau \lrincir{\x_t^\ast - \bar{\x}_{t+1}} - \bar{\x}_t}^2 - \lrnorm{\bar{\x}_{t+1} - \bar{\x}_t}^2 } \\ \nonumber
= & \lrangle{\frac{1}{n}\sum_{i=1}^n \nabla f_{i,t}(\x_{i,t};\xi_{i,t}), \tau \lrincir{\x_t^\ast - \bar{\x}_{t+1}}} + \frac{1}{2\eta}\lrincir{ \lrnorm{\tau \lrincir{\x_t^\ast - \bar{\x}_{t+1}}}^2 + 2\lrangle{\tau \lrincir{\x_t^\ast - \bar{\x}_{t+1}}, \bar{\x}_{t+1} - \bar{\x}_t } }.
\end{align} Note that the above inequality holds for any $0< \tau \le 1$. Divide $\tau$ on both sides, and we have
\begin{align}
\nonumber
I_2(t) = & \EE_{ \Xi_{n,t} \sim \Dcal_{n,t} } \lrangle{\frac{1}{n}\sum_{i=1}^n \nabla f_{i,t}(\x_{i,t};\xi_{i,t}), \bar{\x}_{t+1} - \x_t^\ast} \\ \nonumber 
\le & \frac{1}{2\eta}\EE_{ \Xi_{n,t} \sim \Dcal_{n,t} }\lrincir{ \lim_{\tau \rightarrow 0^+}\tau \lrnorm{\lrincir{\x_t^\ast - \bar{\x}_{t+1}}}^2 + 2\lrangle{ \x_t^\ast - \bar{\x}_{t+1}, \bar{\x}_{t+1} - \bar{\x}_t } } \\ \nonumber
= & \frac{1}{\eta}\EE_{ \Xi_{n,t} \sim \Dcal_{n,t} }\lrangle{ \x_t^\ast - \bar{\x}_{t+1}, \bar{\x}_{t+1} - \bar{\x}_t } \\ \label{equa_I3_temp}
= & \frac{1}{2\eta}\EE_{ \Xi_{n,t} \sim \Dcal_{n,t} }\lrincir{ \lrnorm{\x_t^\ast - \bar{\x}_t}^2 - \lrnorm{\x_t^\ast - \bar{\x}_{t+1}}^2 - \lrnorm{\bar{\x}_t - \bar{\x}_{t+1}}^2 }. 
\end{align} Besides, we have
\begin{align}
\nonumber
& \lrnorm{\x_{t+1}^\ast - \bar{\x}_{t+1}}^2 - \lrnorm{\x_t^\ast - \bar{\x}_{t+1}}^2 \\ \nonumber 
= & \lrnorm{\x_{t+1}^\ast}^2 - \lrnorm{\x_t^\ast}^2 - 2\lrangle{\bar{\x}_{t+1}, -\x_t^\ast + \x_{t+1}^\ast} \\ \nonumber
= & \lrincir{\lrnorm{\x_{t+1}^\ast} - \lrnorm{\x_t^\ast}} \lrincir{\lrnorm{\x_{t+1}^\ast} + \lrnorm{\x_t^\ast}} - 2\lrangle{\bar{\x}_{t+1}, -\x_t^\ast + \x_{t+1}^\ast} \\ \nonumber
\le & \lrnorm{\x_{t+1}^\ast - \x_t^\ast} \lrincir{\lrnorm{\x_{t+1}^\ast} + \lrnorm{\x_t^\ast}} + 2\lrnorm{\bar{\x}_{t+1}} \lrnorm{\x_{t+1}^\ast-\x_t^\ast} \\ \nonumber
\le & 4\sqrt{R}\lrnorm{\x_{t+1}^\ast - \x_t^\ast}.   
\end{align} The last inequality holds due to our assumption, that is, $\lrnorm{\x_{t+1}^\ast}=\lrnorm{\x_{t+1}^\ast - \0}\le \sqrt{R}$, $\lrnorm{\x_t^\ast} = \lrnorm{\x_t^\ast-\0} \le \sqrt{R}$, and $\lrnorm{\bar{\x}_{t+1}} = \lrnorm{\bar{\x}_{t+1}-\0} \le \sqrt{R}$. 

Thus, telescoping $I_2(t)$ over $t\in[T]$, we have 
\begin{align}
\nonumber
\sum_{t=1}^T I_2(t) \le & \frac{1}{2\eta}\EE_{ \Xi_{n,T} \sim \Dcal_{n,T} }\lrincir{ 4\sqrt{R}\sum_{t=1}^T\lrnorm{\x_{t+1}^\ast - \x_t^\ast} + \lrnorm{\bar{\x}_1^\ast - \bar{\x}_1}^2 - \lrnorm{\bar{\x}_T^\ast - \bar{\x}_{T+1}}^2 }  - \frac{1}{2\eta} \EE_{ \Xi_{n,T} \sim \Dcal_{n,T} }\sum_{t=1}^T \lrnorm{\bar{\x}_t - \bar{\x}_{t+1}}^2 \\ \nonumber
\le & \frac{1}{2\eta}\lrincir{ 4\sqrt{R} M + R } - \frac{1}{2\eta} \EE_{ \Xi_{n,T} \sim \Dcal_{n,T} } \sum_{t=1}^T \lrnorm{\bar{\x}_t - \bar{\x}_{t+1} }^2.
\end{align} Here, $M$ the budget of the dynamics.

Combining those bounds of $I_1(t)$, and $I_2(t)$ together, we finally obtain
\begin{align}
\nonumber
& \EE_{ \Xi_{n,T} \sim \Dcal_{n,T} } \sum_{t=1}^T\sum_{i=1}^n f_{i,t}(\x_{i,t};\xi_{i,t}) - f_{i,t}(\x_t^\ast;\xi_{i,t}) \le n \sum_{t=1}^T \lrincir{ I_1(t) + I_2(t) } \\ \nonumber
\le & \eta T \sigma^2 +  \lrincir{L + 2\eta L^2}  \EE_{ \Xi_{n,T} \sim \Dcal_{n,T} }\sum_{t=1}^T\sum_{i=1}^n \lrnorm{ \bar{\x}_t - \x_{i,t} }^2 + \frac{n}{2\eta}\lrincir{ 4\sqrt{R}M + R  } + 2n\eta T G^2 + \frac{nT \eta G^2}{1-\rho}\\ \nonumber
\refabovecir{\le}{\textcircled{1}} & \eta T \sigma^2 +  \frac{\lrincir{2L + 4\eta L^2 }nT\eta^2(G^2 + \sigma^2)}{(1-\rho)^2} + \frac{n}{2\eta}\lrincir{ 4\sqrt{R}M + R  } + \lrincir{2+\frac{1}{1-\rho}}n\eta T G^2.
\end{align}  
$\textcircled{1}$ holds due to Lemma \ref{Lemma_x_variance_norm_square}
\begin{align}
\nonumber
\EE_{ \Xi_{n,T} \sim \Dcal_{n,T} } \sum_{i=1}^n\sum_{t=1}^T \lrnorm{\x_{i,t} - \bar{\x}_t}^2 \le \frac{nT\eta^2 (2G^2+2\sigma^2) }{(1-\rho)^2}.
\end{align} 
Rearranging items, we finally completes the proof.
\end{proof}

\begin{Lemma}
\label{Lemma_average_update_rule}
Denote $\bar{\x}_t = \frac{1}{n}\sum_{i=1}^n \x_{i,t}$. We have
\begin{align}
\nonumber
\bar{\x}_{t+1} =  \bar{\x}_{t} - \eta \lrincir{\frac{1}{n} \sum_{i=1}^n \nabla f_{i,t}(\x_{i,t};\xi_{i,t})}. 
\end{align}
\end{Lemma}
\begin{proof}
Denote
\begin{align}
\nonumber
\X_t = &  [\x_{1,t}, \x_{2,t}, ..., \x_{n,t}] \in \RR^{d\times n}, \\ \nonumber
\G_t = & [\nabla f_{1,t}(\x_{1,t};\xi_{1,t}), \nabla f_{2,t}(\x_{2,t};\xi_{2,t}), ..., \nabla f_{n,t}(\x_{n,t};\xi_{n,t})] \in \RR^{d\times n}.
\end{align}

Denote that 
\begin{align}
\nonumber
\x_{i,t+1} = \sum_{j=1}^n \W_{ij}\x_{j,t} - \eta \nabla f_{i,t}(\x_{i,t};\xi_{i,t}).
\end{align} Equivalently, we re-formulate the update rule as
\begin{align}
\nonumber
\X_{t+1} = \X_{t}\W - \eta \G_t.
\end{align} Since the confusion matrix $\W$ is doublely stochastic, we have
\begin{align}
\nonumber
\W \1 = \1.
\end{align} Thus, we have
\begin{align}
\nonumber
\bar{\x}_{t+1} = & \frac{1}{n}\sum_{i=1}^n \x_{i,t+1} \\ \nonumber
= & \X_{t+1}\frac{\1}{n} \\ \nonumber 
= & \X_{t}\W\frac{\1}{n} - \eta \G_t\frac{\1}{n} \\ \nonumber
= & \X_{t}\frac{\1}{n} - \eta \G_t\frac{\1}{n} \\ \nonumber
=& \bar{\x}_{t} - \eta \lrincir{\frac{1}{n} \sum_{i=1}^n \nabla f_{i,t}(\x_{i,t};\xi_{i,t})}. 
\end{align} It completes the proof.
\end{proof}

\begin{Lemma}
\label{lemma_dsm_norm}
For any doubly stochastic matrix $\W$, its norm $\lrnorm{\W} = 1$.
\end{Lemma}   
\begin{proof}
According to Birkhoff-von Neumann theorem \citep{dufosse:hal-31}, $\W$ is a convex combination of some, e.g. $k$, permutation matrices $\{\M_i\}_{i=1}^k$, that is, 
\begin{align}
\W = \sum_{i=1}^k \zeta_i \M_i,    
\end{align} where $0\le \zeta_i \le 1$ for any $1\le i\le k$ and $\sum_{i=1}^k \zeta_i  = 1$.

For any a vector $\u$ such that $\lrnorm{\u}=1$, we have $\lrnorm{\M_i} = \sup_{\lrnorm{\u}=1}\lrnorm{\M_i \u} = \sup_{\lrnorm{\u}=1} \lrnorm{\u} = 1$. Therefore, we have
\begin{align}
\lrnorm{\W} \le \sum_{i=1}^k \zeta_i \lrnorm{\M_i} = \sum_{i=1}^k \zeta_i =1.
\end{align} Since $\lrnorm{\W}^2$ is the maximal eigenvalue of $\W\W\Tr$, and $\W\W\Tr$ has a eigenvalue $1$, thus $\lrnorm{\W} \ge 1$. 

Therefore, $\lrnorm{\W} = 1$, and the proof is completed.
\end{proof}

\begin{Lemma}
\label{Lemma_hanlin_1}
Denote $\v_1 = \frac{\1}{\sqrt{n}}$.  Given any matrix $\X$ and doubly stochastic matrix  $\W$, we have
\begin{align}
\nonumber
\lrnorm{\X \W^t - \X \v_1 \v_1\Tr }_F^2 \le \lrincir{\rho^t \norm{\X}_F}^2, 
\end{align} where  $\rho = \lrnorm{\W-\V_n}$ and $\V_n = \frac{1}{n}\1\1\Tr$.  
\end{Lemma} 

\begin{proof}
We start from the left hand side:
\begin{align}
\nonumber
\lrnorm{\X\W^t-\X\V_n}_F^2 = & \tr\lrincir{(\X\W^t-\X\V_n) (\X\W^t-\X\V_n)\Tr} \\ \nonumber
= & \tr\lrincir{\X\W^t \lrincir{\W^t}\Tr\X\Tr-\X\W^t \V_n\Tr\X\Tr-\X\V_n\lrincir{\W^t}\Tr\X\Tr + \X\V_n\V_n\Tr\X\Tr} \\ \nonumber
\refabovecir{=}{\textcircled{1}} & \tr\lrincir{\X\W^t \lrincir{\W^t}\Tr\X\Tr-\X\V_n\X\Tr}.
\end{align} `$\tr$' represents the \textit{trace} operator. $\textcircled{1}$ holds due to $\V_n = \V_n\Tr$,  
\begin{align}
\nonumber
\W^t \V_n\Tr = \frac{1}{n}\W^{t-1} \W \1\1\Tr =  \frac{1}{n} \W^{t-1} \1\1\Tr = \cdots = \V_n,
\end{align} similarly
\begin{align}
\nonumber
\V_n\lrincir{\W^t}\Tr = \frac{1}{n}\1\1\Tr\lrincir{\W^t}\Tr =  \frac{1}{n}\1\1\Tr\lrincir{\W^{t-1}}\Tr = \V_n,
\end{align} and $\V_n\V_n\Tr = \V_n$. 

Additionally, since $\W$ is a doubly stochastic matrix, we have
\begin{align}
\nonumber
\W^t \lrincir{\W^t}\Tr \frac{\1}{\sqrt{n}} = \W^t \lrincir{\W\Tr}^{t-1}\W\Tr \frac{\1}{\sqrt{n}} = \W^t \lrincir{\W\Tr}^{t-1} \frac{\1}{\sqrt{n}} = \cdots = \frac{\1}{\sqrt{n}}.
\end{align}  Thus, $1$ is one of eigenvalues of $\W^t \lrincir{\W^t}\Tr$, and its largest eigenvalue $\lambda^{(1)}_{\W^t \lrincir{\W^t}\Tr} \ge 1$.
According to Lemma \ref{lemma_dsm_norm}, we have
\begin{align}
\lambda^{(1)}_{\W^t \lrincir{\W^t}\Tr} = \lrnorm{\W^t \lrincir{\W\Tr}^t} \le \lrnorm{\W \lrincir{\W}\Tr}^t \le   \lrnorm{\W}^t \lrnorm{\W\Tr}^t \le 1. 
\end{align}

Thus, $\lambda^{(1)}_{\W^t \lrincir{\W^t}\Tr} = 1$ is its largest eigenvalue, and $\v_1 = \frac{1}{\sqrt{n}}\1$ is the corresponding eigenvector. Since $\W^t \lrincir{\W^t}\Tr$ is a symmetric and positive semi-definite matrix, we decompose $\W^t \lrincir{\W^t}\Tr$ as $\W^t \lrincir{\W^t}\Tr = \sum_{i=1}^n \lambda^{(i)}_{\W^t \lrincir{\W^t}\Tr} \v_i \v_i\Tr = \P \bLambda \P\Tr$, where $\P = [\v_1, \v_2, ..., \v_n]\in\RR^{n\times n}$. $\v_i$ is the normalized eigenvector corresponding to the $i$-th eigenvalue $\lambda^{(i)}_{\W^t \lrincir{\W^t}\Tr}$. Denote the absolute value of the $i$-th largest eigenvalue of $\W^t \lrincir{\W^t}\Tr$ by $\lambda^{(i)}_{\W^t \lrincir{\W^t}\Tr}$, that is, $\lambda^{(1)}_{\W^t \lrincir{\W^t}\Tr} \ge \cdots \ge \lambda^{(i)}_{\W^t \lrincir{\W^t}\Tr} \ge \cdots \ge \lambda^{(n)}_{\W^t \lrincir{\W^t}\Tr}$.  $\bLambda$ is a diagonal matrix, and $\lambda^{(i)}_{\W^t \lrincir{\W^t}\Tr}$ is its $i$-th element.

Due to  $\W^t \lrincir{\W^t}\Tr = \V_n + \sum_{i=2}^n \lambda^{(i)}_{\W^t \lrincir{\W^t}\Tr} \v_i\v_i\Tr$, we have
\begin{align}
\nonumber
\lrnorm{\X\W^t-\X\V_n}_F^2 = & \tr\lrincir{\X \lrincir{\sum_{i=2}^n \lambda^{(i)}_{\W^t \lrincir{\W^t}\Tr} \v_i\v_i\Tr} \X\Tr} \\ \nonumber
= & \sum_{i=2}^n \lambda^{(i)}_{\W^t \lrincir{\W^t}\Tr} \tr\lrincir{   \X \lrincir{\v_i\v_i\Tr} \X\Tr} \\ \nonumber
= & \sum_{i=2}^n \lambda^{(i)}_{\W^t \lrincir{\W^t}\Tr} \lrnorm{\X \v_i}^2 \\ \nonumber
\le & \lambda^{(2)}_{\W^t \lrincir{\W^t}\Tr}  \sum_{i=1}^n \lrnorm{\X \v_i}^2    \\ \nonumber     
= & \lambda^{(2)}_{\W^t \lrincir{\W^t}\Tr}  \norm{\X \P}_F^2 {~~~~~~~~~~~~} \text{(due to $\P$'s definition)}    \\ \nonumber
= & \lambda^{(2)}_{\W^t \lrincir{\W^t}\Tr}  \tr\lrincir{\X \P \P\Tr \X\Tr} \\ \nonumber
= & \lambda^{(2)}_{\W^t \lrincir{\W^t}\Tr}  \tr\lrincir{\X \X\Tr}  \\ \nonumber 
= & \lambda^{(2)}_{\W^t \lrincir{\W^t}\Tr}  \norm{\X}_F^2.
\end{align}

Recall that $\V_n = \V_n\Tr$, $\W^t \V_n\Tr = \V_n\lrincir{\W^t}\Tr = \V_n$, and $\V_n\V_n = \V_n$. We thus have 
\begin{align}
\nonumber
& (\W - \V_n)^t (\W\Tr - \V_n)^t = (\W - \V_n)^{t-1} (\W\W\Tr - \W\V_n -\V_n\W\Tr + \V_n \V_n) (\W\Tr - \V_n)^{t-1} \\ \nonumber 
= & (\W - \V_n)^{t-1} (\W\W\Tr - \V_n) (\W\Tr - \V_n)^{t-1}\\ \nonumber 
= & \cdots = \W^t (\W\Tr)^t - \V_n. 
\end{align}  Since $\rho = \lrnorm{\W - \V_n} = \lrnorm{\W\Tr - \V_n}$, we have 
\begin{align}
\nonumber
\lambda^{(2)}_{\W^t \lrincir{\W^t}\Tr} = \lrnorm{\W^t (\W\Tr)^t - \V_n} \le \lrnorm{\W - \V_n}^t \lrnorm{\W\Tr - \V_n}^t = \rho^{2t}. 
\end{align} We finally have
\begin{align}
\nonumber
\lrnorm{\X\W^t-\X\V_n}_F^2 \le \lrincir{\rho^t \norm{\X}_F}^2.
\end{align} It completes the proof.
\end{proof}

\begin{Lemma}[Lemma $6$ in \citep{Tang:2018un}]
\label{Lemma_hanlin_2}
Given two non-negative sequences $\{a_t\}_{t=1}^{\infty}$ and $\{b_t\}_{t=1}^{\infty}$ that satisfying
\begin{align}
\nonumber
a_t = \sum_{s=1}^t \rho^{t-s} b_s,
\end{align} with $\rho \in [0,1)$, we have
\begin{align}
\nonumber
\sum_{t=1}^k a_t^2 \le \frac{1}{(1-\rho)^2}\sum_{s=1}^k b_s^2.
\end{align}
\end{Lemma}

\citet{8015179Shahram} investigates the dynamic regret of DOG, and provide the following sublinear regret.
\begin{Theorem}[Implied by Theorem $3$ and Corollary $4$ in \citet{8015179Shahram}]
\label{theorem_privious_dog_regret}
Choose $\eta = \sqrt{\frac{(1-\rho) M}{T}}$ in Algorithm \ref{algo_DOG}. Under Assumption \ref{assumption_bounded_gradient_domain}, the dynamic regret $\Rcal_T^{\textsc{DOG}}$ is bounded by $\Ocal{n^{\frac{3}{2}}\sqrt{\frac{MT}{1-\rho}} }$.
\end{Theorem}

As illustrated in Theorem \ref{theorem_privious_dog_regret},   \citet{8015179Shahram} has provided a $\Ocal{n\sqrt{nTM}}$ regret for DOG. Comparing with the regret in \citet{8015179Shahram}, our analysis improves the dependence on $n$, which benefits from the following better bound of difference between $\x_{i,t}$ and $\bar{\x}_t$.
\begin{Lemma}
\label{Lemma_x_variance_norm_square}
Setting $\eta>0$ in Algorithm \ref{algo_DOG}, under Assumption \ref{assumption_bounded_gradient_domain} we have 
\begin{align}
\nonumber
\EE_{ \Xi_{n,T} \sim \Dcal_{n,T} } \sum_{i=1}^n\sum_{t=1}^T \lrnorm{\x_{i,t} - \bar{\x}_t}^2 \le \frac{2nT\eta^2 (G^2+\sigma^2) }{(1-\rho)^2}.
\end{align}
\end{Lemma}
\begin{proof}
Denote that 
\begin{align}
\nonumber
\x_{i,t+1} = \sum_{j=1}^n \W_{ij}\x_{j,t} - \eta \nabla f_{i,t}(\x_{i,t};\xi_{i,t}), 
\end{align} and according to Lemma \ref{Lemma_average_update_rule}, we have 
\begin{align}
\nonumber
\bar{\x}_{t+1} = \bar{\x}_t - \eta \lrincir{\frac{1}{n}\sum_{i=1}^n \nabla f_{i,t}(\x_{i,t};\xi_{i,t})}.
\end{align} Denote
\begin{align}
\nonumber
\X_t = &  [\x_{1,t}, \x_{2,t}, ..., \x_{n,t}] \in \RR^{d\times n}, \\ \nonumber
\G_t = & [\nabla f_{1,t}(\x_{1,t};\xi_{1,t}), \nabla f_{2,t}(\x_{2,t};\xi_{2,t}), ..., \nabla f_{n,t}(\x_{n,t};\xi_{n,t})] \in \RR^{d\times n}.
\end{align} By letting $\x_{i,1} = \0$ for any $i\in[n]$, the update rule is re-formulated as 
\begin{align}
\nonumber
\X_{t+1} = \X_t \W - \eta \G_t = - \sum_{s=1}^t \eta \G_s \W^{t-s}. 
\end{align} Similarly, denote $\bar{\G}_t = \frac{1}{n}\sum_{i=1}^n \nabla f_{i,t}(\x_{i,t};\xi_{i,t})$, and we have
\begin{align*}
\bar{\x}_{t+1} = \bar{\x}_t - \eta \lrincir{\frac{1}{n}\sum_{i=1}^n \nabla f_{i,t}(\x_{i,t};\xi_{i,t})} = - \sum_{s=1}^t \eta \bar{\G}_s. 
\end{align*}

Therefore, we obtain
\begin{align}
\nonumber
\sum_{i=1}^n \lrnorm{\x_{i,t} - \bar{\x}_t}^2 \refabovecir{=}{\textcircled{1}} & \sum_{i=1}^n \lrnorm{ \sum_{s=1}^{t-1} \lrincir{\eta \bar{\G}_s - \eta \G_s \W^{t-s-1}\e_i} }^2   \\ \nonumber
\refabovecir{=}{\textcircled{2}} & \lrnorm{ \sum_{s=1}^{t-1} \lrincir{\eta \G_s\v_1 \v_1\Tr - \eta \G_s \W^{t-s-1}} }^2_F   \\ \nonumber
= & \eta^2 \lrincir{\lrnorm{ \sum_{s=1}^{t-1} \lrincir{\G_s\v_1 \v_1\Tr - \G_s \W^{t-s-1}} }_F }^2   \\ \nonumber
\le & \eta^2 \lrincir{\sum_{s=1}^{t-1}\lrnorm{ \lrincir{\G_s\v_1 \v_1\Tr - \G_s \W^{t-s-1}} }_F }^2   \\ \nonumber
\refabovecir{\le}{\textcircled{3}} & \eta^2 \lrincir{\sum_{s=1}^{t-1} \rho^{t-s-1} \lrnorm{\G_s}_F }^2   \\ \nonumber
\end{align} $\textcircled{1}$ holds due to $\e_i$ is a unit basis vector, whose $i$-th element is $1$ and other elements are $0$s. $\textcircled{2}$ holds due to $\v_1 = \frac{\1_n}{\sqrt{n}}$. $\textcircled{3}$ holds due to Lemma \ref{Lemma_hanlin_1}.

Thus, we  have
\begin{align}
\nonumber
& \EE_{ \Xi_{n,T} \sim \Dcal_{n,T} } \sum_{i=1}^n\sum_{t=1}^T \lrnorm{\x_{i,t} - \bar{\x}_t}^2 \\ \nonumber 
\le & \EE_{ \Xi_{n,T} \sim \Dcal_{n,T} } \sum_{t=1}^T \lrincir{ \sum_{s=1}^{t-1} \eta \rho^{t-s-1} \lrnorm{\G_s}_F}^2  \\ \nonumber
\refabovecir{\le}{\textcircled{1}} & \frac{\eta^2}{(1-\rho)^2} \EE_{ \Xi_{n,T} \sim \Dcal_{n,T} } \lrincir{  \sum_{t=1}^T \lrnorm{\G_t}_F^2 } \\ \nonumber
= & \frac{\eta^2}{(1-\rho)^2} \lrincir{ \EE_{ \Xi_{n,T} \sim \Dcal_{n,T} } \sum_{t=1}^T \sum_{i=1}^n  \lrnorm{\nabla f_{i,t}(\x_{i,t};\xi_{i,t}) - \nabla F_{i,t}(\x_{i,t}) + \nabla F_{i,t}(\x_{i,t})}^2 } \\ \nonumber
\le & \frac{2\eta^2}{(1-\rho)^2} \EE_{ \Xi_{n,T} \sim \Dcal_{n,T} } \sum_{t=1}^T \sum_{i=1}^n  \lrnorm{\nabla f_{i,t}(\x_{i,t};\xi_{i,t}) - \nabla F_{i,t}(\x_{i,t})}^2 + \frac{2\eta^2}{(1-\rho)^2} \EE_{ \Xi_{n,T} \sim \Dcal_{n,T} } \sum_{t=1}^T \sum_{i=1}^n  \lrnorm{\nabla F_{i,t}(\x_{i,t})}^2 \\ \nonumber
\le & \frac{nT\eta^2 (2G^2+2\sigma^2) }{(1-\rho)^2}.
\end{align} $\textcircled{1}$ holds due to Lemma \ref{Lemma_hanlin_2}.  It completes the proof.
\end{proof}

\textbf{Proof to Theorem \ref{theorem_local_models_closer}:}
\begin{proof}
Setting $\eta = \sqrt{\frac{(1-\rho) \lrincir{nM\sqrt{R} + nR}}{ nTG^2 + T\sigma^2 }}$ into Lemma \ref{Lemma_x_variance_norm_square}, we finally complete the proof.
\end{proof}

\section*{Supplementary materials for empirical studies}

\begin{figure}[!]
\setlength{\abovecaptionskip}{0pt}
\setlength{\belowcaptionskip}{0pt}
\centering 
\subfigure{\includegraphics[width=0.5\columnwidth]{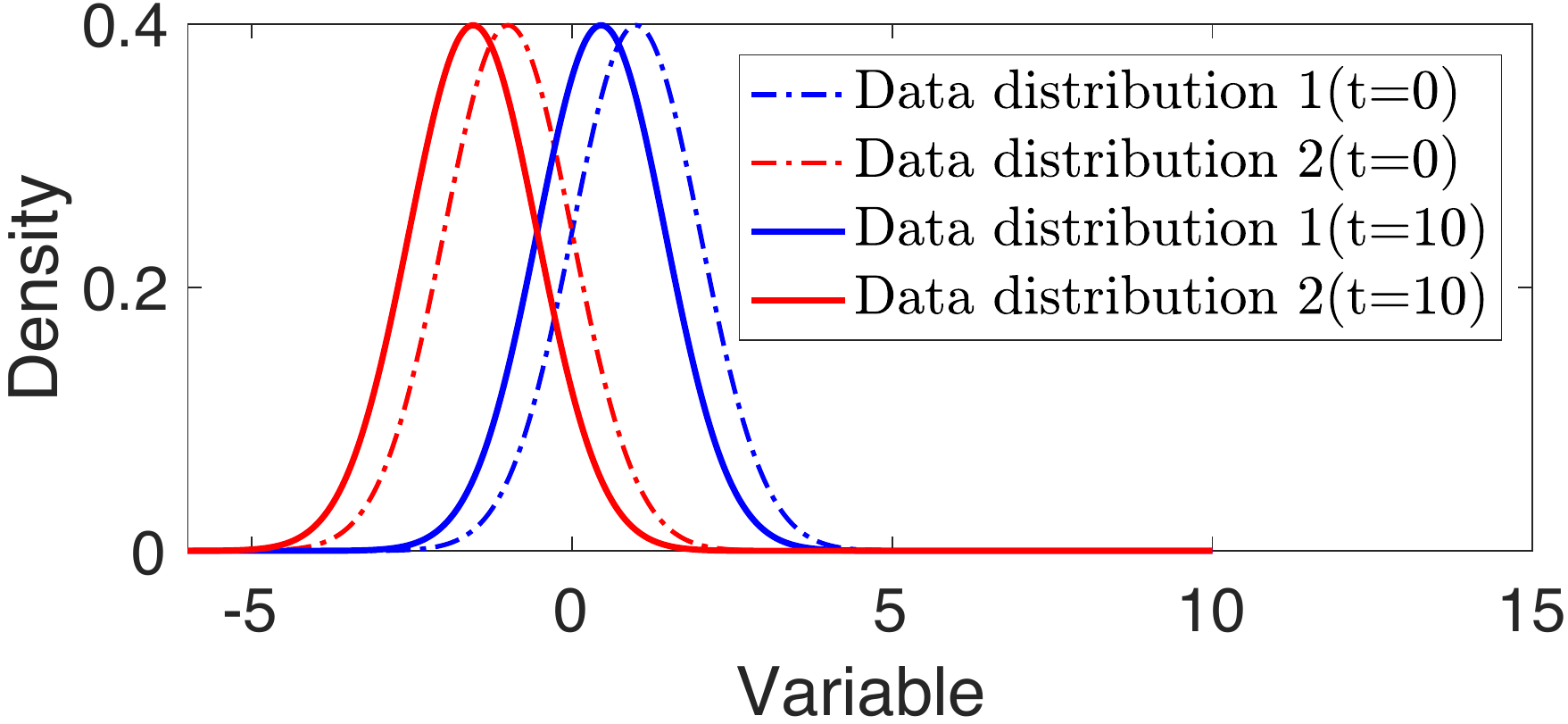}\label{figure_dynamics}}
\caption{An illustration of the dynmaics caused by the time-varying distributions of data. Data distributions $1$ and $2$ satisify $N(1+\sin(t), 1)$ and $N(-1+\sin(t), 1)$, respectively.  Suppose we want to conduct classification between data drawn from distributions $1$ and $2$, respectively. The optimal classification model should change over time.}
\label{figure_illus_dynamics}
\end{figure}

The dynamics of time-varying distributions are illustrated in Figure \ref{figure_illus_dynamics}, which shows the change of the optimal learning model over time and the importance of studying the dynamic regret. We use such time-varying distribution to simulate the dynamics of the synthetic data.

More numerical results are presented as Figures \ref{figure_compare_loss_others}-\ref{figure_compare_topology_others}.

\begin{figure*}[!t]
\setlength{\abovecaptionskip}{0pt}
\setlength{\belowcaptionskip}{0pt}
\centering 
\subfigure[\textit{room-occupancy}, $5$ nodes, ring topology]{\includegraphics[width=0.33\columnwidth]{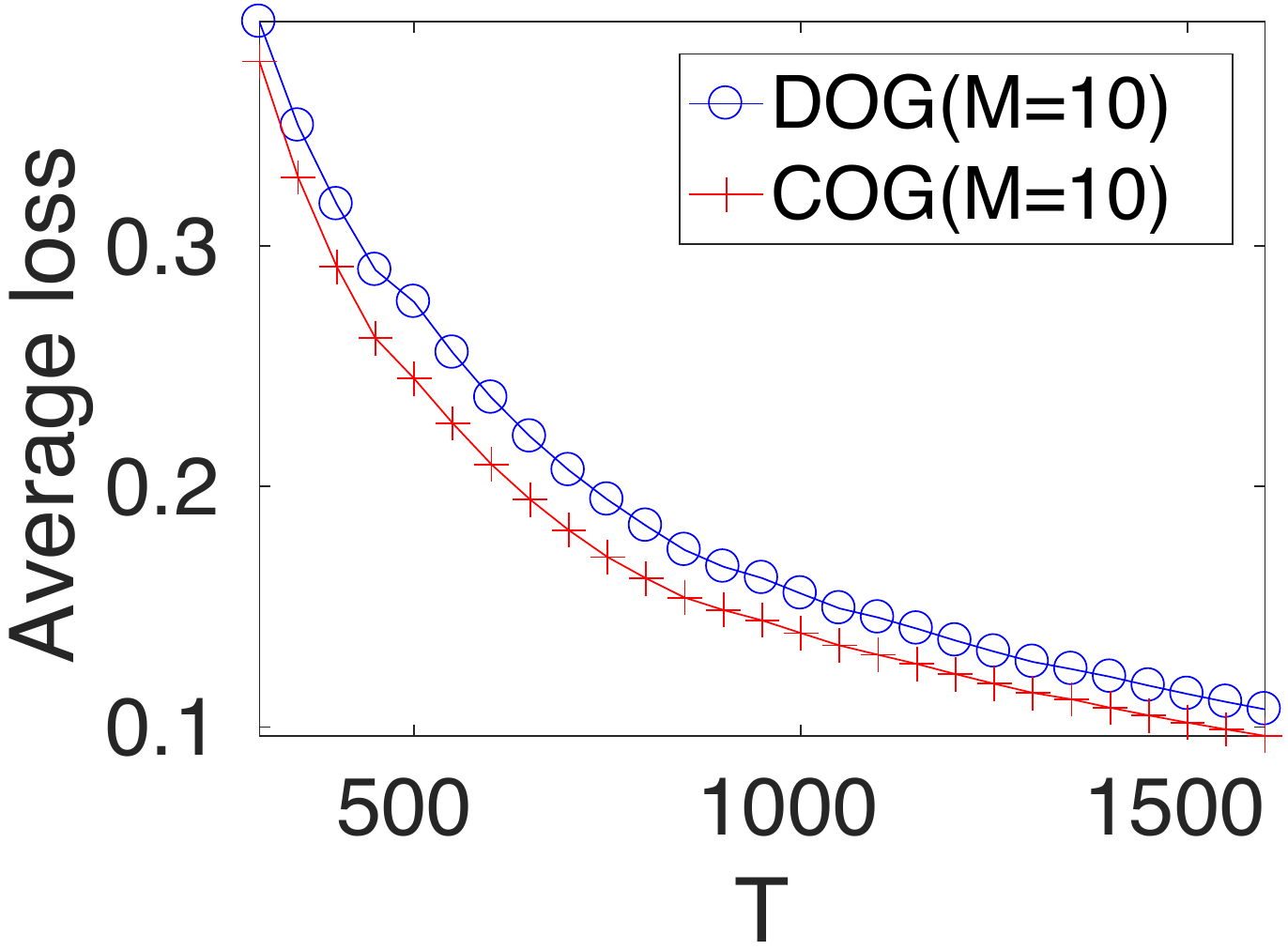}\label{figure_ave_loss_iteration_occupancy}}
\subfigure[\textit{usenet2}, $5$ nodes, ring topology]{\includegraphics[width=0.32\columnwidth]{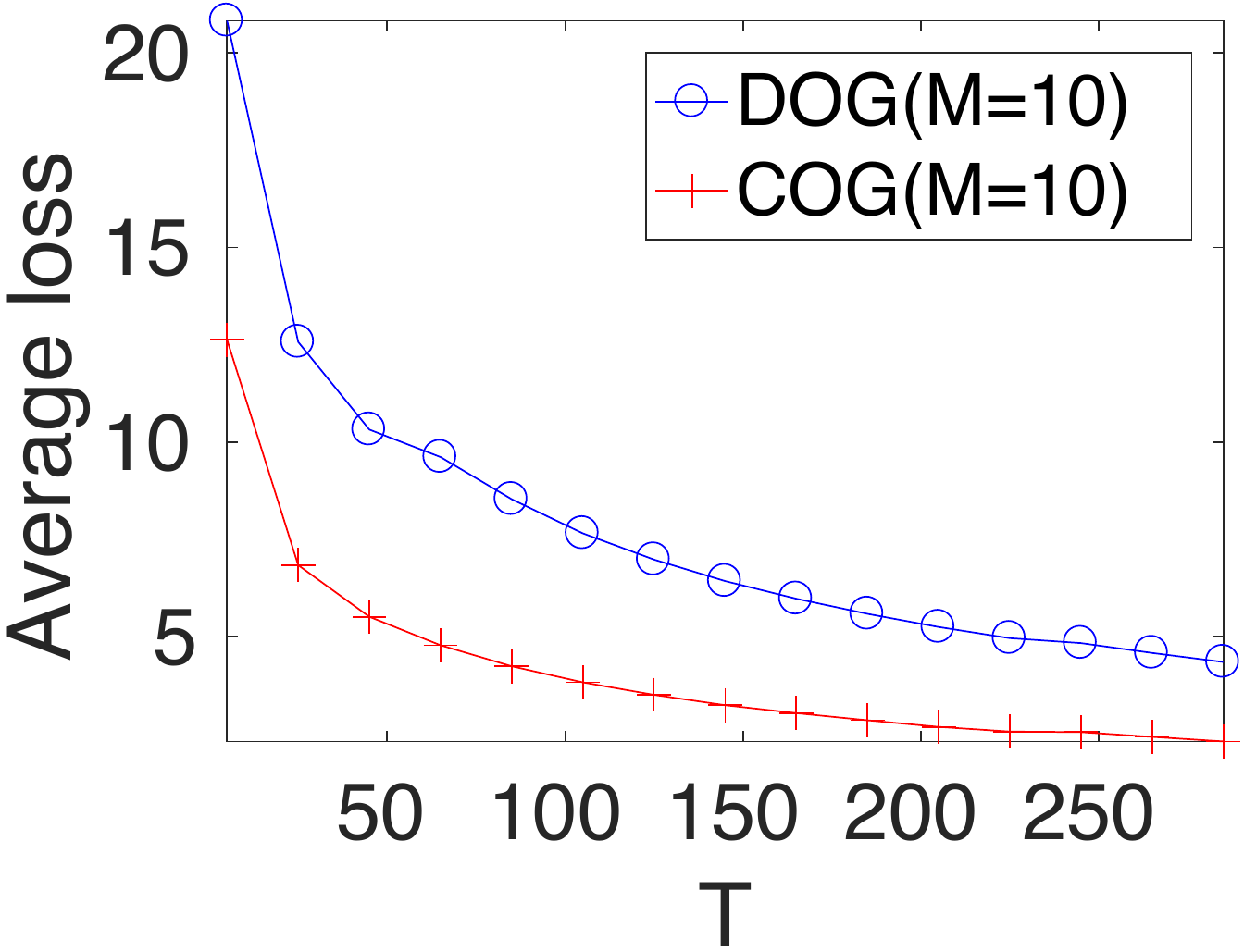}\label{figure_ave_loss_iteration_usenet2}}
\subfigure[\textit{spam}, $5$ nodes, ring topology]{\includegraphics[width=0.32\columnwidth]{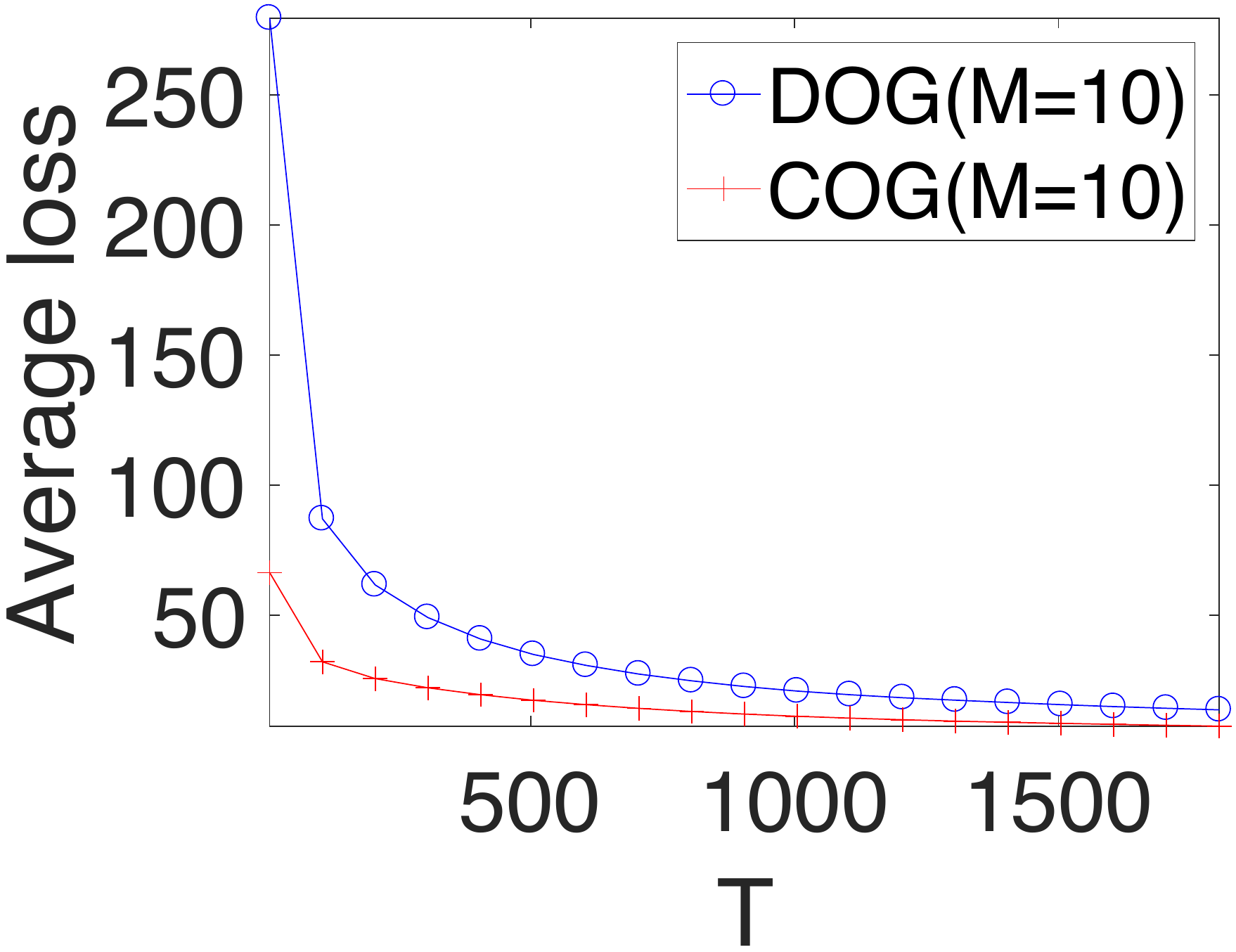}\label{figure_ave_loss_iteration_spam}}
\caption{The average loss yielded by DOG is comparable to that yielded by COG.}
\label{figure_compare_loss_others}
\end{figure*}

\begin{figure*}[!t]
\setlength{\abovecaptionskip}{0pt}
\setlength{\belowcaptionskip}{0pt}
\centering 
\subfigure[\textit{room-occupancy}, ring topology]{\includegraphics[width=0.32\columnwidth]{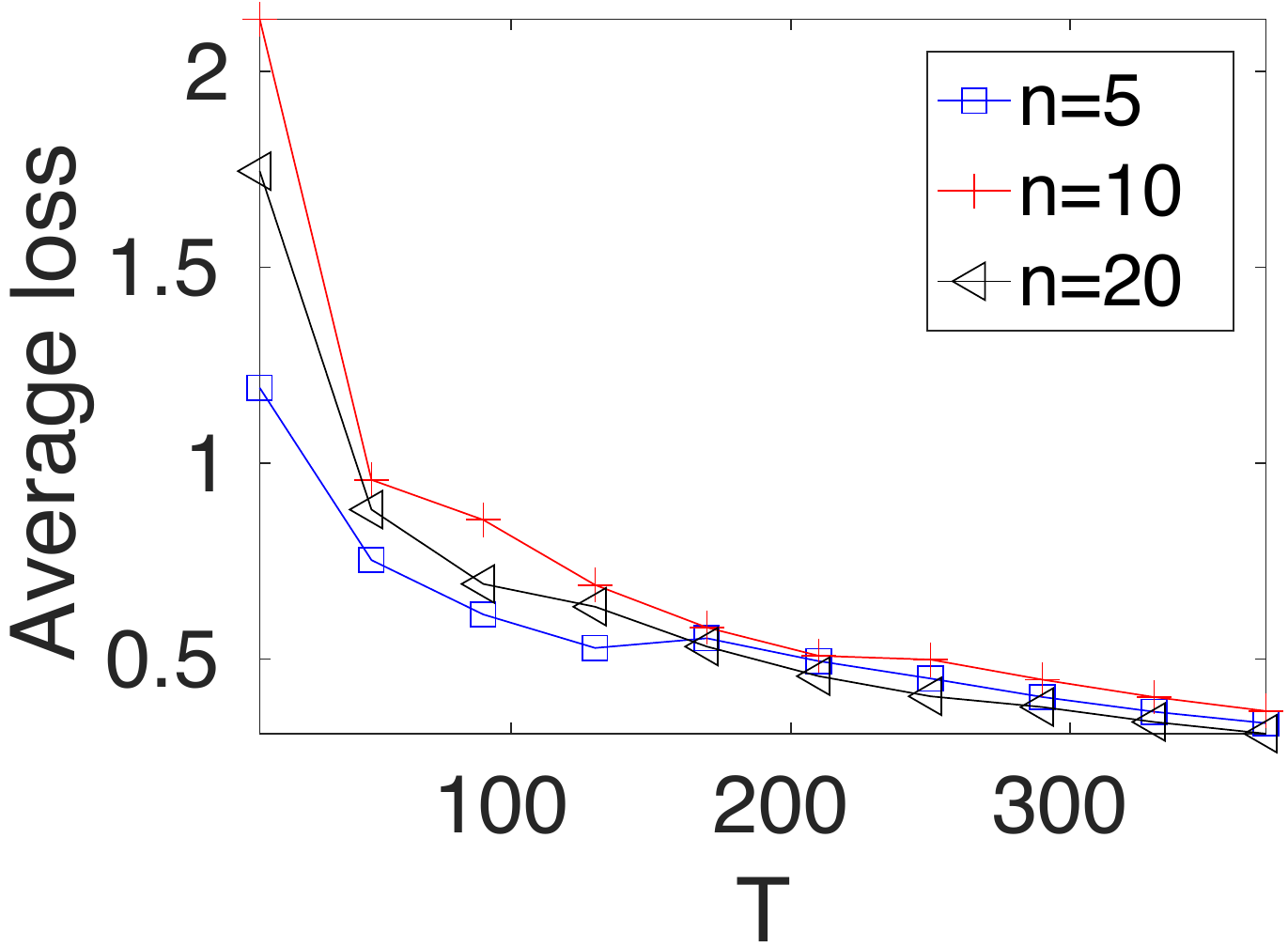}\label{figure_ave_loss_network_size_occupancy}}
\subfigure[\textit{usenet2}, ring topology]{\includegraphics[width=0.32\columnwidth]{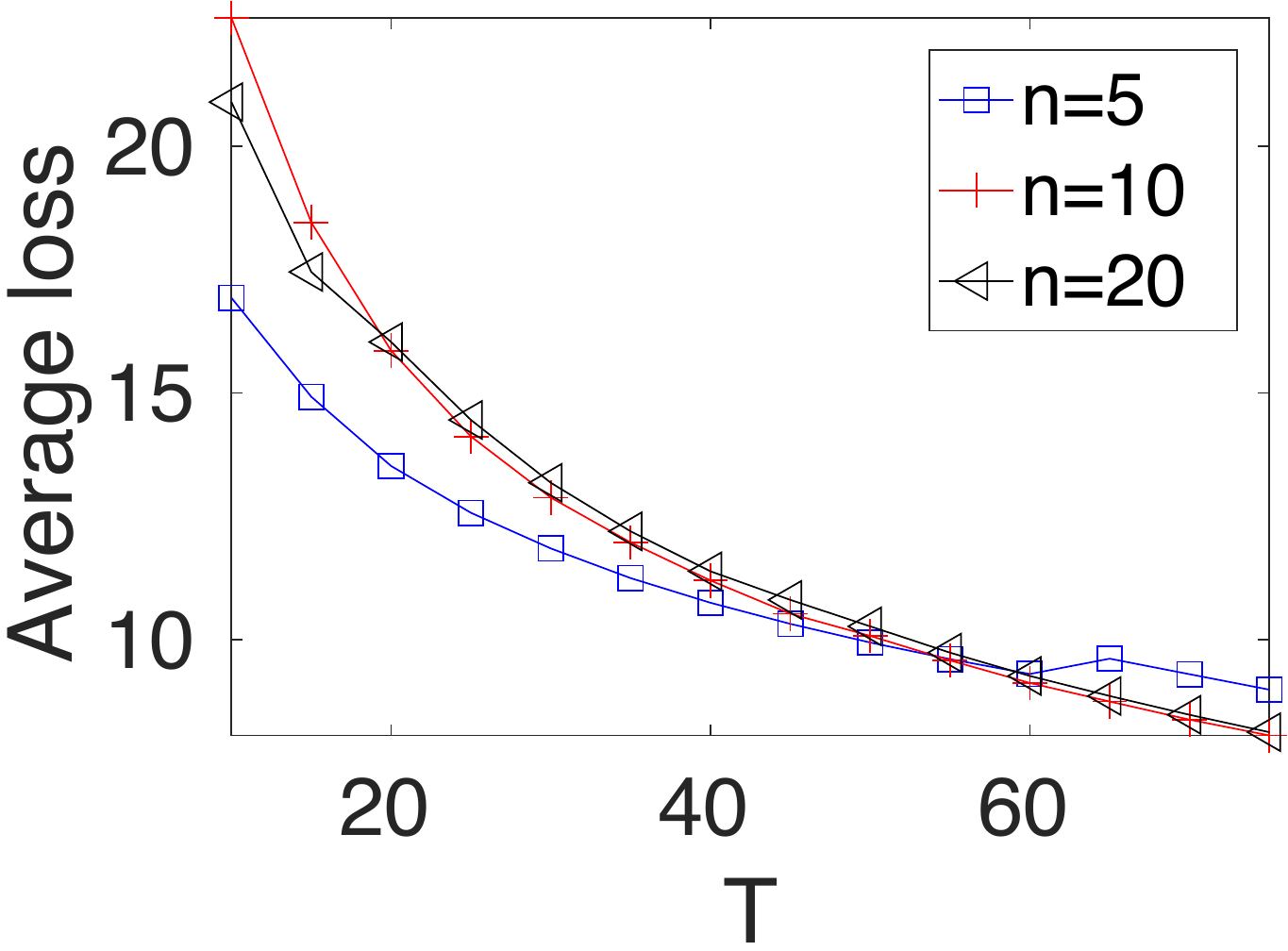}\label{figure_ave_loss_network_size_usenet2}}
\subfigure[\textit{spam}, ring topology]{\includegraphics[width=0.32\columnwidth]{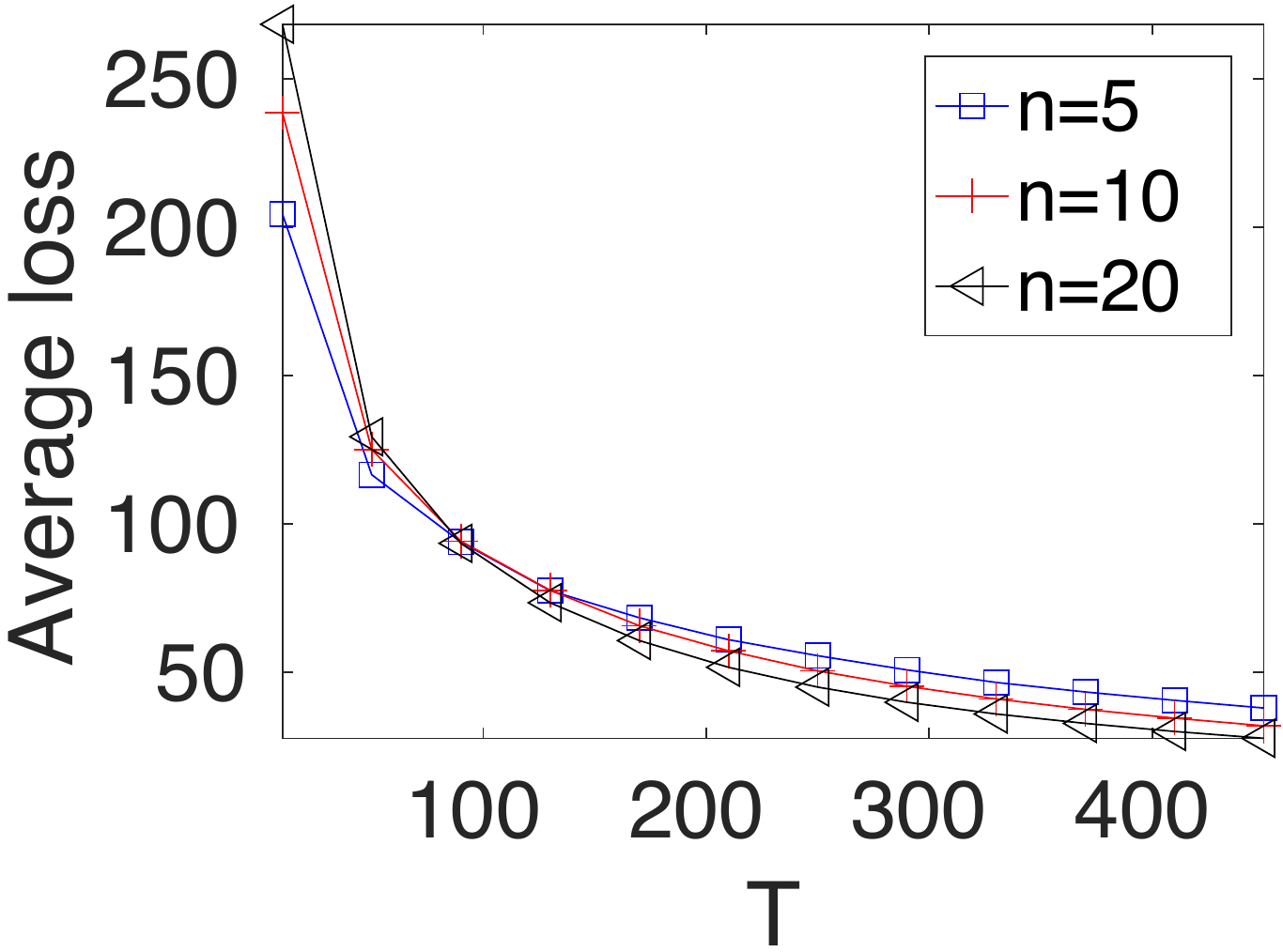}\label{figure_ave_loss_network_size_spam}}
\caption{The average loss yielded by DOG is insensitive to the network size.}
\label{figure_compare_network_size_others}
\end{figure*}

\begin{figure*}[!t]
\setlength{\abovecaptionskip}{0pt}
\setlength{\belowcaptionskip}{0pt}
\centering 
\subfigure[\textit{room-occupancy}, $20$ nodes]{\includegraphics[width=0.32\columnwidth]{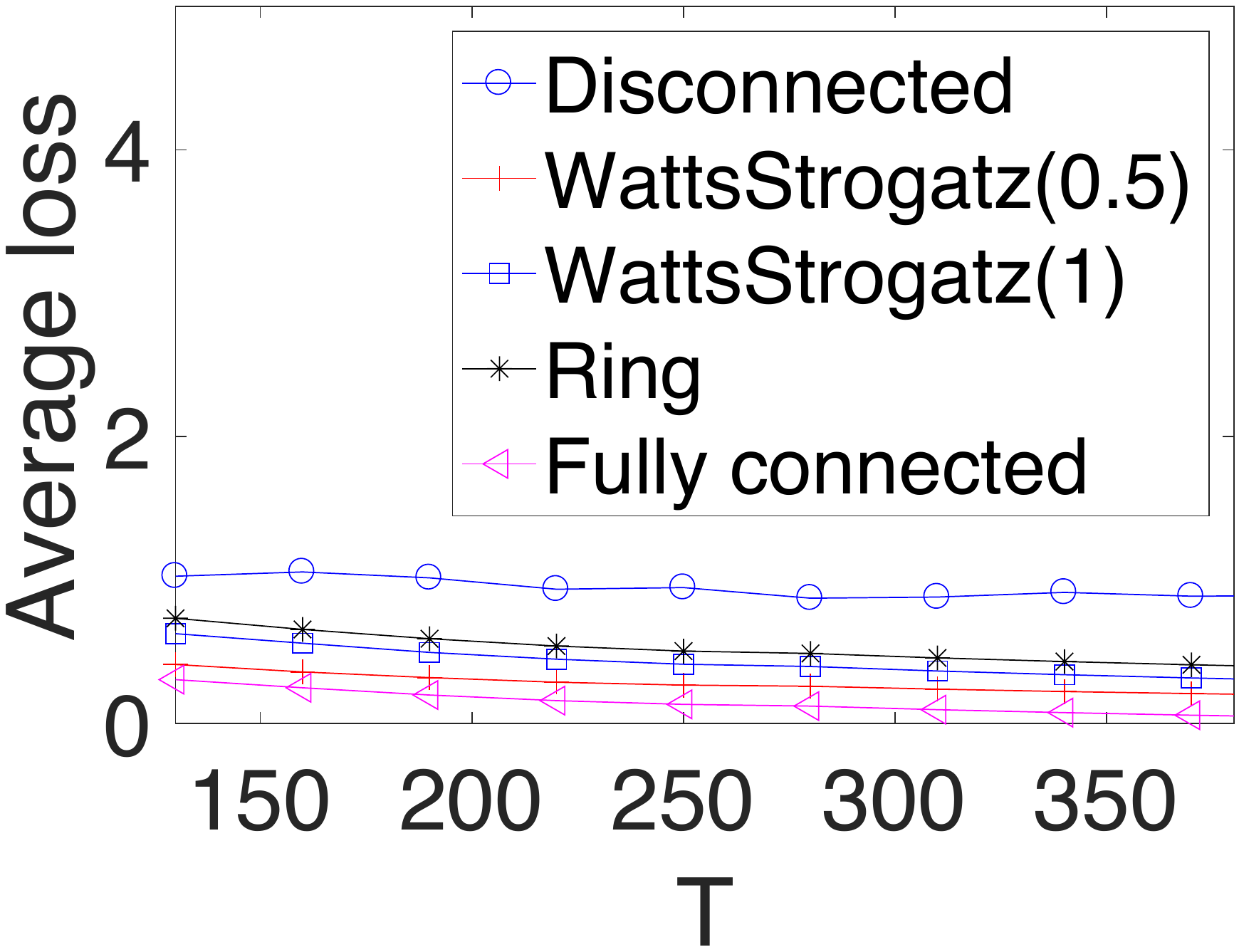}\label{figure_ave_loss_topology_occupancy}}
\subfigure[\textit{usenet2}, $20$ nodes]{\includegraphics[width=0.32\columnwidth]{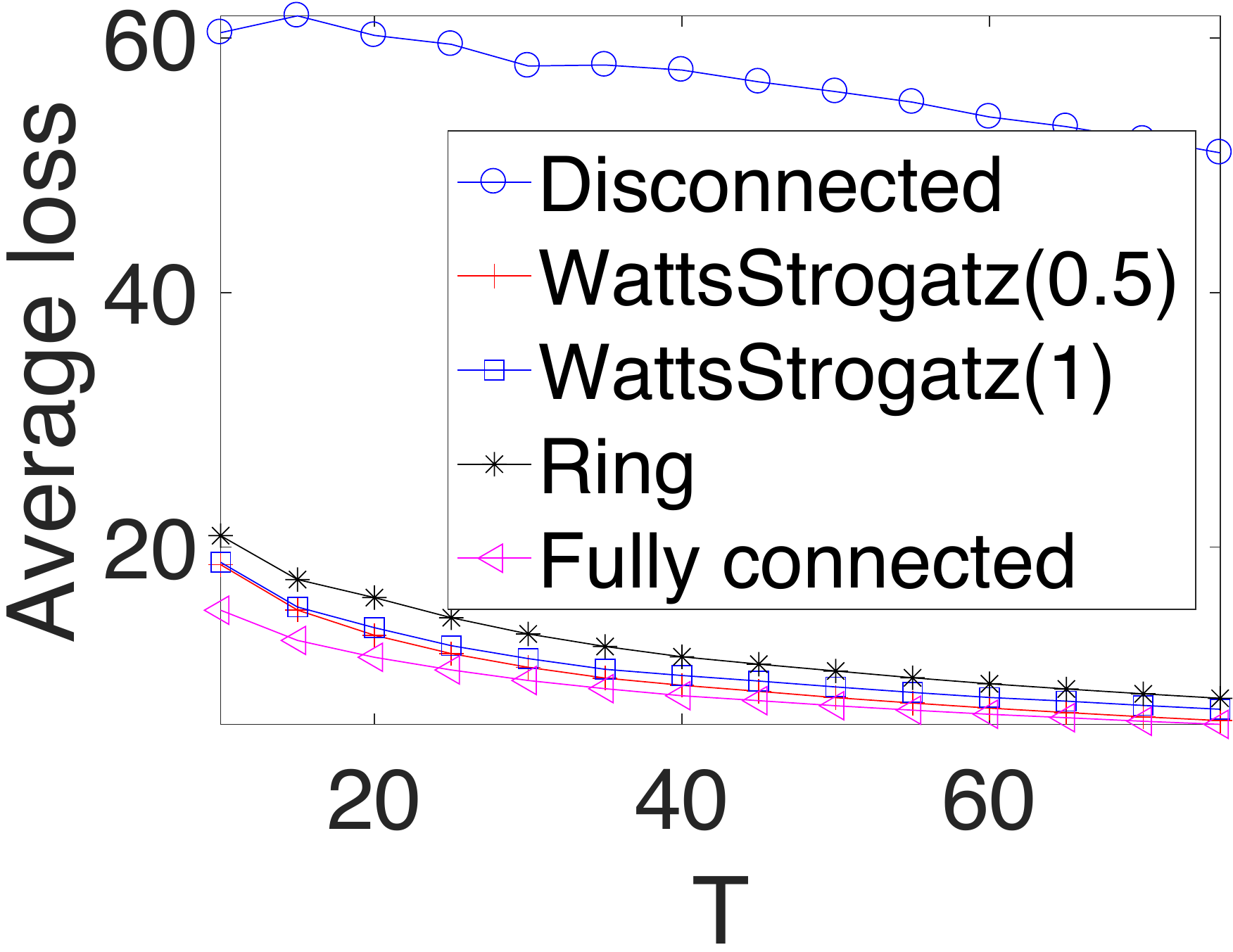}\label{figure_ave_loss_topology_occupancy}}
\subfigure[\textit{spam}, $20$ nodes]{\includegraphics[width=0.32\columnwidth]{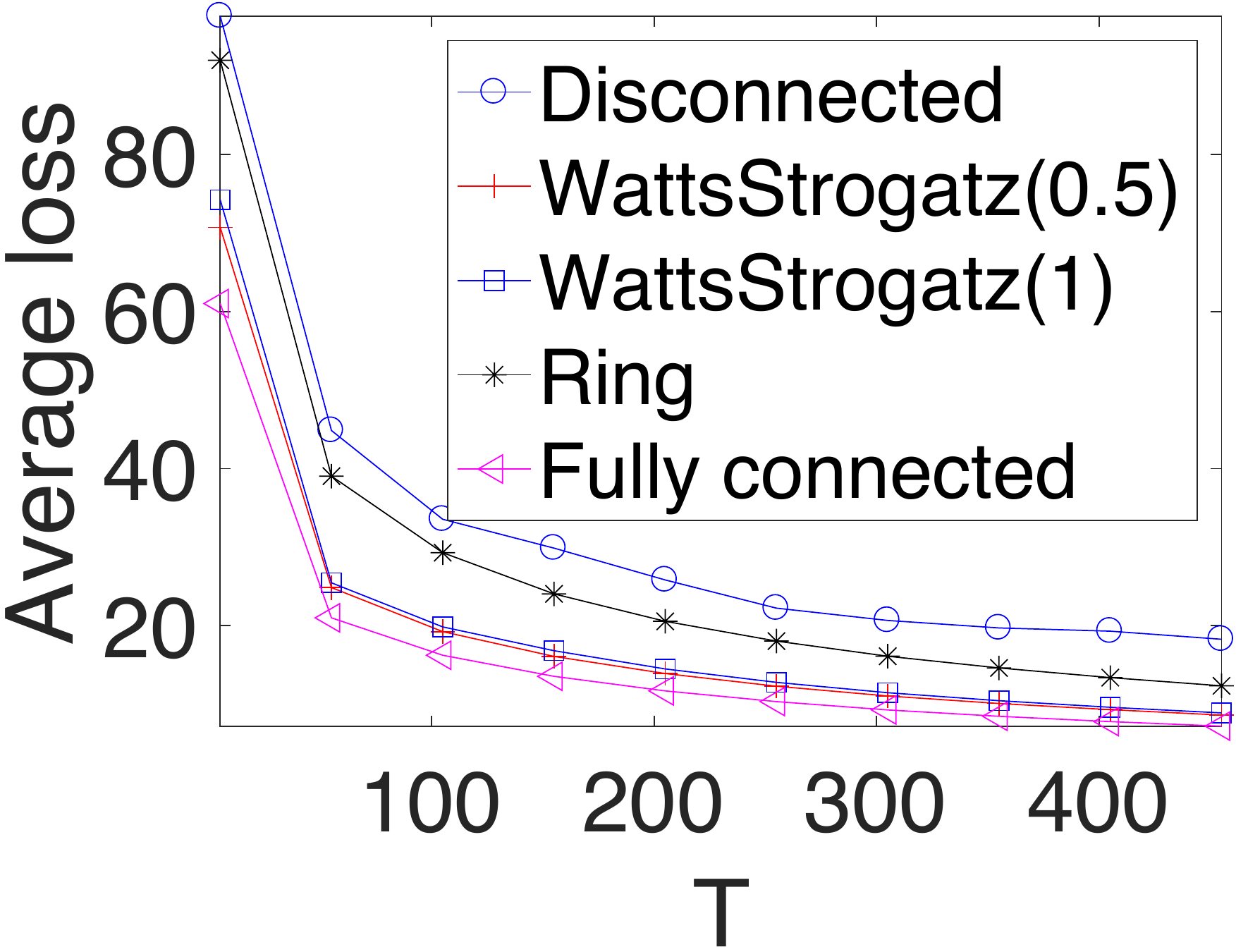}\label{figure_ave_loss_topology_spam}}
\caption{The average loss yielded by DOG is insensitive to the topology of the network.}
\label{figure_compare_topology_others}
\end{figure*}

\end{document}